\documentclass[12pt]{article}

\usepackage[utf8]{inputenc} 
\usepackage[truedimen,margin=1.2in]{geometry} 
\usepackage[T1]{fontenc}    
\usepackage{hyperref}       
\usepackage{booktabs}       
\usepackage{amsfonts}       
\usepackage{nicefrac}       
\usepackage{microtype}      
\usepackage{url}            
\usepackage{amsmath}
\usepackage{amssymb}
\usepackage{amsthm}
\usepackage{graphicx} 
\usepackage{fancyhdr}
\usepackage{times}
\usepackage{algorithm} 
\usepackage{algorithmic} 
\usepackage{subfigure} 
\usepackage{multirow}
\usepackage{here}
\usepackage{wrapfig}
\usepackage{centernot}
\usepackage{enumitem}
\usepackage{authblk}
\usepackage{breakcites}

\newtheorem{theorem}{Theorem}
\newtheorem{lemma}{Lemma}
\newtheorem{proposition}{Proposition}

\newtheorem{assumption}{Assumption}
\newtheorem{example}{Example}

\usepackage{titlesec}
\titleformat*{\section}{\large\bfseries}
\titleformat*{\subsection}{\normalsize\bfseries}

\allowdisplaybreaks[4]

\def\pd<#1>{\langle #1 \rangle}

\begin{document}
\title{\Large{Stochastic Particle Gradient Descent for Infinite Ensembles}}

\author[1,2]{Atsushi Nitanda \thanks{atsushi\_nitanda@mist.i.u-tokyo.ac.jp}}
\author[1,2,3]{Taiji Suzuki \thanks{taiji@mist.i.u-tokyo.ac.jp}}
\affil[1]{\normalsize{Graduate School of Information Science and Technology, The University of Tokyo}}
\affil[2]{\normalsize{Center for Advanced Intelligence Project, RIKEN}}
\affil[3]{\normalsize{PRESTO, Japan Science and Technology Agency}}

\date{}

\maketitle

\begin{abstract}
The superior performance of ensemble methods with infinite models are well known.
Most of these methods are based on optimization problems in infinite-dimensional spaces with some regularization, 
for instance, boosting methods and convex neural networks use $L^1$-regularization with the non-negative constraint.
However, due to the difficulty of handling $L^1$-regularization, these problems require early stopping or a rough approximation to solve it inexactly.
In this paper, we propose a new ensemble learning method that performs in a space of probability measures,
that is, our method can handle the $L^1$-constraint and the non-negative constraint in a rigorous way.
Such an optimization is realized by proposing a general purpose stochastic optimization method for learning probability measures via parameterization using transport maps on base models. 
As a result of running the method, a transport map to output an {\it infinite ensemble} is obtained, which forms a residual-type network.
From the perspective of functional gradient methods, we give a convergence rate as fast as that of a stochastic optimization method for finite dimensional nonconvex problems.
Moreover, we show an {\it interior optimality property} of a local optimality condition used in our analysis. 
\end{abstract}


\sloppy

\section{Introduction}
The goal of the binary classification problem is to find a measurable function, called a classifier, from the feature space to the range $[-1,1]$,
which is required to minimize the expected classification error.
The ensemble, including boosting and bagging, is one method used to solve this problem, by constructing a complex classifier by combining base classifiers.
It is well-known empirically that such a classifier attains good generalization performance in experiments and applications \cite{bk1999,die2000,vj2001}.

Several studies explain the generalization ability of ensembles.
The first important result was presented by \cite{sfbl1998}, where the margin theory for convex combinations of classifiers was introduced in the context of boosting,
which provides a bound on the expected classification error by the empirical distribution of the margin.
The slightly tighter generalization bound was shown in \cite{kp2002}, by using the complexities of the function class such as the covering numbers and Rademacher complexity.
Moreover, in the same paper, it was shown that the generalization bound can be further improved under suitable conditions.

These analyses imply that the ensemble that minimizes the empirical margin distribution or the empirical risk function for a sufficiently large dataset has a good generalization ability.
Such a classifier is usually obtained via an optimization method for the empirical risk minimization problem, e.g., AdaBoost \cite{freund1997decision}, LogitBoost \cite{friedman2000additive},
Arc-gv \cite{bre1999}, AdaBoost$^*_\nu$ \cite{rw2005}, $\alpha$-Boost \cite{friedman2001greedy}, and AnyBoost \cite{mbbf2000}.
These methods are based on the strategy of coordinate descent methods; a base classifier is chosen and its weight is decided in some way such as the line search in each iteration.
This iteration is performed in the space of linear combinations of base classifiers, although the generalization bounds are only provided for convex combinations.
Therefore, these methods need a regularization technique to prevent the rapid growth of its $L^1$-norm, such as early stopping with small learning rates \cite{rzh2004,sf2012}.

Although a complex and powerful classifier is required to achieve high classification accuracy for difficult tasks, the early stopping sometimes interrupts obtaining such a classifier.
In this paper, we propose a new ensemble learning method, called {\it Stochastic Particle Gradient Descent} (SPGD), based on a completely different strategy from those of the existing methods.
The SPGD performs in a space of probability measures on a set of continuously parameterized base classifiers and constructs an ensemble by the expectation with respect to the obtained probability measure.
In other words, our ideal method potentially handles ensembles of an infinite number of base classifiers and we can derive a practical variant of the method by approximating it with finite particles for any desired smoothness.
This is in opposition to the strategy of existing methods that successively increase the number of basis to be combined.
This difference produces some advantages, that is, there is no need to impose penalization in the method, and we are free from both handling the penalization term and adjusting early stopping timing; moreover, our method can find a complex ensemble quickly.

We call such classifiers, combined by probability measures, the {\it infinite ensemble}.
Since the existing generalization bounds are provided for finite or countable combinations, we first extend these results to infinite ensembles and provide almost the same bounds.
Generalization bounds are composed of the empirical margin distribution and the complexity terms.
Clearly, infinite ensembles have greater ability to reduce the former term compared to that obtained by traditional ensemble methods, and hence our method can lead to a more powerful classifier.

Moreover, we present the convergence analyses of the SPGD method whose update is realized by pushing-forward a probability measure by transport maps on the space of base classifiers.
Since this update can be regarded as an extension of stochastic gradient descent (SGD) in a finite-dimensional space, we can explore the properties of our method by analogy with SGD.
To make it rigorous, we provide several theoretical tools. 
Especially, a counterpart of Taylor's formula allows us to derive a local optimality condition of problems and to construct convergence analyses of the method. 
Indeed, using this formula, we present a convergence rate of the method as fast as that of a stochastic optimization method for finite dimensional nonconvex problems.
Moreover, we show the {\it interior optimality property} where a probability measure $\mu_*$ that possesses a continuous density function and satisfies a local optimality condition is optimal in
the support of itself under appropriate assumptions on the support.
This property is inherent in problems with respect to a probability measure and its proof mainly relies on partial differential equation theory.

Furthermore, we provide two practical variants of SPGD method.
One is a natural approximation of a transport map in SPGD using finite particles, and we note this approximation forms a residual-type network \cite{he2016deep}.
The other is a more practical variant without resampling of particles, and we show this variant can be regarded as well-initialized SGD for the nonweighted voting classification problem,
that is, we can say it is an extension of the vanilla SGD to the method for optimizing a general probability measure.

\subsection*{Contributions}
\begin{itemize}
\item We derive the generalization bounds on the infinite ensemble by extending existence results and we propose a stochastic optimization method for learning a probability measure.
  Since our method performs in the space of probability measures, it directly minimizes the loss function to obtain an infinite ensemble without the early stopping.
\item We present a local optimality condition and its properties, especially the interior optimality property is important and inherent in the problem of learning probability measures.
  This property guarantees the optimality of the obtained probability measure in its support under some conditions.
\item We reveal the relation between our method and the vanilla SGD in the finite-dimensional space and we provide the convergence analysis by using the traditional optimization theory.
  Moreover, we present several aspects of the method that lead to deeper understanding of the method, specifically connections with discretization of the gradient flow in a space of probability measures 
  and the functional gradient method in the $L^2$-space.
\end{itemize}
    
\subsection*{Related Work}
Ensemble learning with infinite models have been received a lot of attention due to their superior performance and many optimization methods have been exploited.
Representative methods are boosting methods \cite{sfbl1998}, convex (continuous) neural networks \cite{bengio2006convex,le2007continuous,rosset2007l1},
and Bayesian neural networks \cite{mackay1992practical,mackay1995probable,neal2012bayesian}.
Kernel methods using shift-invariant kernels \cite{rahimi2008random} also combine a basis, although a base probability measure to sample base functions is pre-determined.
Most of these methods are based on optimization problems in infinite-dimensional spaces with some regularization,
for instance, boosting methods and convex neural networks use the $L^1$-regularization with the non-negative constraint, that is, combinations by probability measures,
and kernel methods with shift-invariant kernels use RKHS-norm regularization which is written by the $L^2$-regularization using an associated probability measure
like infinite-layer networks \cite{livni2017learning}.
$L^2$-regularized problems in kernel methods can be efficiently solved by the functional gradient descent \cite{kivinen2004online,dai2014scalable,1512.02394,dieuleveut2016nonparametric} or methods using explicit random features
\cite{rahimi2008random,livni2017learning}.
Note that the random kitchen sinks \cite{rahimi2009weighted} adopt the $L^\infty$-constraint rather than the $L^2$-regularization.
As for the $L^1$-regularization, combining its good generalization performance \cite{kp2002,kpl2003,bach2014breaking} with the fact
that the $L^1$-ball always includes the $L^2$ ($L^\infty$)-ball, superior classification performance is expected in many cases.
However, solving $L^1$-regularized problems are more challenging than $L^2$ ($L^\infty$)-regularized problems because handling $L^1$-regularization is difficult from the optimization perspective
and so these problems usually require early stopping or some approximation to solve it inexactly as pointed out above, though our method can handle these constraints naturally.

From the perspective of optimizing the probability measure, gradient-based Bayesian inference methods \cite{welling2011bayesian,dai2016provable,lw2016} are related to ours.
Especially, {\it stochastic variational gradient descent} (SVGD) proposed in \cite{lw2016} is most related to our work, which has a similar flavor to our method.
Convergence analysis and gradient flow perspective were given in \cite{liu2017stein} and further analysis was provided in \cite{1711.10927}.
However, while SVGD is a method specialized to minimize the Kullback--Leibler-divergence based on Stein's identity technique, our method does not require special structure of a loss function; 
hence, our method can be applied to a wider class of problems and theoretical results hold in the more general setting, though we focus our study only on the ensemble learning.
We would like to remark an interesting point of our method compared to the normalizing flow \cite{rezende2015variational} that approximates Bayes posterior through deep neural networks.
In our method, a transport map is obtained by stacking residual-type layers \cite{he2016deep} iteratively, 
hence a residual network to output an infinite ensemble is built naturally.

\section{Infinite Ensembles and Generalization Bounds}
In this section, we extend the well-known generalization bounds obtained in \cite{kp2002} to infinite ensembles.
We first precisely define this classifier.
Let us denote the Borel measurable feature space and the label set by $\mathcal{X}\subset \mathbb{R}^n$ and $\mathcal{Y}=\{-1,1\}$, respectively.
Let $\Theta \subset \mathbb{R}^d$ be a Borel set and $\mathcal{H}=\{h_\theta:\ \mathcal{X} \rightarrow [-1,1];\ {\rm Borel}\ {\rm measurable} \mid \theta \in \Theta\}$ be a subset of base binary classifiers that is parameterized by $\theta \in \Theta$.
We sometimes use the notation $h(\theta,x)$ to denote $h_\theta(x)$ when it is regarded as a function with respect to $\theta$ and $x$.
We assume that $h(\theta,x)$ is Borel measurable on $\Theta \times \mathcal{X}$ and continuous with respect to $x \in \mathcal{X}$. 

\begin{example}[Linear Classifier] \label{ex_linear_classifier}
For $\theta\in \Theta$, we define the linear classifier as follows:
\[ h_\theta(x) = {\rm tanh}(\theta^\top x), \]
which separates the feature space $\mathcal{X}$ linearly using a hyperplane with normal vector $\theta$.
\end{example}
\begin{example}[Neural Network] 
Let $\{l_s\}_{s=0}^L$ be the sizes of layers, where $l_0=n,\ l_L=1$.
We set $d=\sum_{s=0}^{L-1}l_sl_{s+1}$ and $\Theta=\prod_{s=0}^{L-1}\Theta_s$, where $\Theta_s \subset \mathbb{R}^{l_sl_{s+1}}$.
For $\theta=\{\theta_s\}_{s=0}^{L-1} \in \Theta$, we define the classifier, that is, an $L$-layer neural network:
\[ h_\theta(x)=\tanh(\theta_L^\top\sigma(\theta_{L-1}^\top\sigma(\cdots\sigma(\theta_1^\top x)\cdots))), \]
where $\sigma$ is a continuous activation function.
\end{example}

Let us denote the set of all Borel probability measures on $\Theta$ by $\mathcal{P}$.
For $\mu \in \mathcal{P}$, we define the infinite ensemble $h_\mu: \mathcal{X}\rightarrow [-1,1]$ as follows: for $x\in \mathcal{X}$,
\[ h_\mu(x)\overset{\mathrm{def}}{=}\mathbb{E}_\mu[h(\theta,x)], \]
where $\mathbb{E}_{\mu}$ is the expectation with respect to $\mu$, and predict the label of $x$ by $\mathrm{sign}(h_\mu(x))$.
Let $\mathcal{G}$ be the set of all infinite ensembles: $\mathcal{G}=\{h_\mu\mid \mu\in \mathcal{P}\}$.
We denote the true underlying Borel probability measure on $\mathcal{X}\times\mathcal{Y}$ by $\mathbb{P}_D$.
The goal of the classification problem under our setting is to find an infinite ensemble $h_\mu \in \mathcal{G}$ providing a small expected classification error:
$\mathbb{P}_D[Yh_\mu(X)\leq 0]$, where $X$ and $Y$ denote random variables taking values in $\mathcal{X}$ and $\mathcal{Y}$, respectively.

\subsection{Generalization Bounds}
To derive the generalization bound, we make an assumption on the growth of the covering numbers of $\mathcal{G}$ as in \cite{kp2002,kp2005}.
Let $m$ denote any discrete probability measure on $\mathcal{X}$ and $\|\cdot\|_{L^2(m)}$ denote the $L^2$-norm with respect to $m$.
For a set of bounded functions $\mathcal{J}$, let $N_m(\epsilon,\mathcal{J})$ be the (external) covering numbers, that is, the minimal number of $\|\cdot\|_{L^2(m)}$-open balls of radius $\epsilon>0$ needed to cover $\mathcal{J}$. The centers of these balls need not belong to $\mathcal{J}$.
Let us make the following common assumption.

\begin{assumption} [Growth of the covering numbers] \label{covering_numbers_assumption}
There is a constant $C>1$ and $V>0$ such that $\forall \epsilon \in (0,1)$, 
\[ \sup_{m}N_m(\epsilon,\mathcal{H}) \leq C\epsilon^{-V}.\]
\end{assumption}


Under the above assumption, we can derive the counterpart of the margin bound in \cite{kp2002}. 
For a finite training dataset $S$, let $\mathbb{P}_S$ denote an empirical measure defined by $S$ and $\mathbb{E}_S$ denote the expectation with respect to $\mathbb{P}_S$.

\begin{theorem} [Margin bound] \label{generalization_thm_1}
Suppose Assumption \ref{covering_numbers_assumption} holds.
Let $N\in \mathbb{N}$ be the number of data.
There exists a constant $K>0$ depending only on $C$ such that for $\forall \delta \in (0,1)$ with probability at least $1-\delta$ over the random choice of $S$ for $\forall h_\mu \in \mathcal{G}$
we have 
\begin{equation*}
  \mathbb{P}_D[Yh_\mu(X)\leq 0] \leq \sqrt{\frac{2\log(1/\delta)}{N}}
  + \inf_{\rho \in (0,1]} \biggl(\mathbb{P}_S[Yh_\mu(X)\leq \rho] + \frac{K}{\rho}\sqrt{\frac{V}{N}}\biggr). 
\end{equation*}
\end{theorem}

Next, we further extend the improved generalization bound in \cite{kp2002} to the infinite ensemble as follows.

\begin{theorem} [Improved margin bound] \label{generalization_thm_2}
Suppose Assumption \ref{covering_numbers_assumption} holds and $\mathcal{X}$ is compact.
Let $N\in \mathbb{N}$ be the number of data.
There exists a constant $K>0$ such that for $\forall \delta \in (0,1)$ with probability at least $1-\delta$ over the random choice of $S$
for $\forall h_\mu \in \mathcal{G}$ we have 
\begin{equation*}
  \mathbb{P}_D[Yh_\mu(X)\leq 0] \leq \frac{K}{N}\log\left(\frac{1}{\delta}\right)
  + K \inf_{\rho\in(0,1]}\left(\mathbb{P}_S[Yh_\mu(X)\leq \rho] + \rho^{\frac{-V}{V+1}}N^{\frac{-(V+2)}{2(V+1)}} \right). 
\end{equation*}
\end{theorem}

Theorem \ref{generalization_thm_1} and \ref{generalization_thm_2} state that the infinite ensembles providing the small empirical margin distribution for a sufficiently large dataset yields a good expected classification error.
This minimization is done via minimizing the empirical risk defined by the convex surrogate loss function such as the exponential loss.
We present the theoretical justifications for this procedure by using the {\it smooth margin} \cite{rsd2004}.
For $\alpha$ and $\mu \in \mathcal{P}$, we define it as follows:
\[ \psi_\alpha(\mu)=-\alpha \log \frac{1}{N}\sum_{j=1}^N\exp \left(-\frac{y_jh_\mu(x_j)}{\alpha} \right). \]

The following proposition states that the smooth margin is a good proxy for the margin distribution and justifies minimizing the exponential loss function.






\begin{theorem} \label{margin_distribution_bound_prop}
We assume $\psi_\alpha(\mu)>0$ for $\alpha>0$ and $\mu \in \mathcal{P}$.
Then it follows that for $0<\forall \rho < \psi_\alpha(\mu)$,
\begin{equation} 
 \mathbb{E}_S[{\bf 1}[Yh_{\mu}(X)\leq \rho]] \leq \frac{\exp\left( (1-\psi_\alpha(\mu))/\alpha\right)-1}{\exp\left( (1-\rho)/\alpha\right)-1}.  \label{margin_distribution_bound}
\end{equation}
\end{theorem}


These observations indicate that an optimization method that can handle infinite ensembles leads to superior generalization performance because its powerful representation ability significantly reduces the empirical margin distribution via the empirical risk minimization compared to base classifiers in $\mathcal{H}$.

\section{Optimization Problem}
In this section, we describe the problem for the infinite ensemble learning and reveal its properties.
Let $l: \mathbb{R}\rightarrow \mathbb{R}$ be a loss function such as the exponential loss.
Let us consider solving the following problem:
\begin{equation}
\underset{\mu\in \mathcal{P}}{\mathrm{min}}\ \mathcal{L}_S(\mu) \overset{\mathrm{def}}{=} \frac{1}{N}\sum_{j=1}^Nl \left(-y_jh_\mu(x_j)\right). \label{abst_prob}
\end{equation}
One way to make this infinite dimensional optimization problem computationally tractable is to parameterize its subspace locally by a space of actions, which may also be infinite-dimensional manifold.
Basically, our proposed method sequentially updates a Borel probability measure on $\Theta$ based on the theory of transportation. 
That is, the current probability measure $\mu_k$ is updated through pushing-forward by a transport map having the form $id+\xi_k$ toward a direction reducing the objective function $\mathcal{L}_S$.
Repeating this procedure, we finally obtain a composite transport map $\phi_T = (id+\xi_{T-1})\circ \cdots (id+\xi_0)$ from an initial probability measure $\mu_0$ and the corresponding probability measure 
$\mu_T$ is obtained by pushing-forward $\mu_0$ by $\phi_T$.
In practice, the final probability measure $\mu_T$ is approximated by samples obtained through $\phi_T$ as $\phi_T(\theta_i) \sim \mu_T$ where $\theta_i \sim \mu_0$ and this approximation makes
the method feasible.
The resulting problem is how to choose $\xi_k$ to optimize (\ref{abst_prob}) and an answer to this question is by using the functional gradient.
To explain our proposed method correctly and to describe the optimization domain with its properties, the following notions are needed.
\begin{itemize}
\item The {\it transport map} is used to describe the proposed method and the optimization domain.
\item The {\it integral probability metric} is used to derive the (local) optimality condition of the problem and topological properties of the optimization domain.
\end{itemize}


\subsection{Optimization Domain and Topological Property}
We set $\Theta=\mathbb{R}^d$.
Let $\mu$ denote any Borel probability measure on $\Theta$ with finite second moment $\mathbb{E}_{\mu}[\|\theta\|^2] < +\infty$ and $\mathcal{P}_2$ denote the set of such probability measures.
We denote by $L^2(\mu)$ the space of $L^2(\mu)$-integrable maps from ${\rm supp}(\mu) \subset \Theta$ into $\Theta$, equipped with $\pd<\cdot,\cdot>_{L^2(\mu)}$-inner product: for $\forall \xi, \forall \zeta \in L^2(\mu)$,
\begin{equation*}
  \pd<\xi,\zeta>_{L^2(\mu)} = \mathbb{E}_{\mu}[\xi(\theta)^\top\zeta(\theta)].
\end{equation*}
In general, for a probability measure $\mu$, the push-forward measure $\phi_\sharp \mu$ by a map $\phi \in L^2(\mu)$ is defined as follows:
for a Borel measurable set $\mathcal{A} \subset \Theta$, 
\begin{equation}
  \phi_\sharp\mu(\mathcal{A}) \overset{\mathrm{def}}{=} \mu(\phi^{-1}(\mathcal{A})).
\end{equation}
For a probability measure $\mu_q$ having a continuous density function $q$, i.e., $d\mu_q = q(\theta)d\theta$, the push-forward measure is described as follows: by the change of variables formula
\[ d\phi_\sharp \mu_q(\theta) = q(\phi^{-1}(\theta)) |\det \nabla \phi(\theta)^{-1}|d\theta. \]
When we use maps in $L^2(\mu)$ to push-forward probability measures, we call these as transport maps.
We can clearly see $\phi_\sharp \mu$ is also contained in $\mathcal{P}_2$ when $\mu \in \mathcal{P}_2$ and see $id$ is contained in $L^2(\mu)$ for arbitrary $\mu \in \mathcal{P}_2$.
Let us consider approximately solving the problem (\ref{abst_prob}) on $\mathcal{P}_2$ by updating transport maps iteratively.
To discuss the local behavior of the problem, we must specify the topology of $\mathcal{P}$; hence, we need to introduce more notions.

\subsection*{Integral Probability Metric on $\mathcal{P}$}
We introduce a kind of integral probability metrics \cite{CIS-137913} on $\mathcal{P}$.
For a positive constant $C>0$, let $f$ be a function on $\Theta$ such that it is uniformly bounded $|f(\theta)|\leq C$ and $f(\theta)$ is $C$-Lipschitz continuous on $\Theta$ with respect to the Euclidean norm.
We denote by $\mathcal{F}_C$ the set of such functions and the subscript will be omitted for simplicity.
This set of functions is used for defining the norm $\|\cdot\|_{\mathcal{F}}$ on the space of linear functionals on $\mathcal{F}$, which includes $\mathcal{P}$.
Specifically, $\|\mu\|_{\mathcal{F}}=\sup_{f\in\mathcal{F}}|\mu(f)|$ for a finite signed measure $\mu$ on $\Theta$, where we denote the integral of a function $f$ with respect to $\mu \in \mathcal{P}$ by $\mu(f)$.
Thus, $\mathcal{P}$ is a metric space with respect to the uniform distance $d_{\mathcal{F}}(\mu,\nu)=\|\mu-\nu\|_{\mathcal{F}}$ for $\mu, \nu \in \mathcal{P}$.
The convergence $d_{\mathcal{F}}(\mu_t,\mu)\rightarrow 0$ is none other than the uniform convergence of integrals $\mu_t(f)\rightarrow \mu(f)$ on $\mathcal{F}$.
Note that this norm defines the same topology as the Dudley metric \cite{dud1968}.

To investigate the local behavior of $\mathcal{L}_S(\mu)$, we need to clarify the continuity of several quantities depending on $\mu\in \mathcal{P}$ and $h(\cdot,x)$.
Especially, $-l'(-yh_\mu(x))y\nabla_\theta h(\theta,x)$ is really important because it is used to describe an optimality condition and performs as the stochastic gradient in the function space.
For simplicity, we use the notation
\begin{equation}
  s_\mu(\theta,x,y)=-l'(-yh_\mu(x))y\nabla_\theta h(\theta,x) \label{stochastic_gradient}
\end{equation}
for this map.
We now make the following assumption and provide the continuity proposition.

\begin{assumption} [Continuity] \label{continuity_assumption}
The set $\{h(\cdot,x)\mid x\in \mathcal{X} \}$ and $\{ \|\mathbb{E}_S[s_{\mu}(\cdot,x,y)]\|_2^2 \mid \mu\in \mathcal{P}\}$ are included in the set $\mathcal{F}_C$.
\end{assumption}

\begin{proposition}\label{continuity_prop_on_P}
Suppose Assumption \ref{continuity_assumption} holds.
Then, $h_\mu(x)$, $\mathcal{L}_S(\mu)$, and $\|\mathbb{E}_S[s_{\mu}(\theta,x,y)]\|_{L^2(\mu)}$ are continuous as a function of $\mu$ on $\mathcal{P}$ with respect to $\|\cdot\|_{\mathcal{F}}$.
\end{proposition}

The next proposition supports the validity of this assumption for Example \ref{ex_linear_classifier} of the linear classifier. 

\begin{proposition} \label{smoothness_proposition}
Let the loss function $l$ be a $\mathcal{C}^1$-class function.
If $h(\cdot,x)$ is two times continuously differentiable and $h(\cdot,x), \nabla_{\theta} h(\cdot,x)$ are Lipschitz continuous with respect to $\theta$ with the same constant
for all $x \in \mathcal{X}$,
then for sufficiently large $C>0$, Assumption \ref{continuity_assumption} is satisfied.
\end{proposition}


\subsection{Local Optimality Condition}
In this subsection, we establish local optimality conditions for the approximated problem of (\ref{abst_prob}) over $\mathcal{P}_2$.
To achieve this goal, we need not only the continuity propositions, but also the counterpart of Taylor's formula in the Euclidean space, giving the intuition to construct and analyze an optimization method for solving the problem.
Thus, we show such a proposition under the following assumption.

\begin{assumption} [Smoothness and boundedness] \label{smooth_bound_assumption}
Let $l$ be a $\mathcal{C}^2$-function and let $h$ be a $\mathcal{C}^2$-function with respect to $\theta$.
Moreover, we assume $\nabla_\theta h(\theta,x)$, $\nabla^2_\theta h(\theta,x)$, and the eigenvalues of the latter matrix are uniformly bounded on $\Theta \times \mathcal{X}$.
\end{assumption}

Note that this assumption also holds for Example \ref{ex_linear_classifier} of the linear classifier under the compactness of $\mathcal{X}$ as in the case of Assumption \ref{continuity_assumption}.
The following is the counterpart of Taylor's formula.

\begin{proposition} \label{asymptotic_equality_prop}
Suppose Assumption \ref{smooth_bound_assumption} holds.
For $\forall \mu \in \mathcal{P}_2$ and $\forall \xi \in L^2(\mu)$,
$\mathcal{L}_S((id+\xi)_\sharp \mu)$ can be represented as follows:
\begin{equation}
\mathcal{L}_S((id+\xi)_\sharp \mu)=\mathcal{L}_S(\mu)+\mathbb{E}_\mu[\mathbb{E}_S[s_\mu(\theta,x,y)]^\top\xi(\theta) ] 
+ \frac{1}{2}H_\mu(\xi) + o(\|\xi\|_{L^2(\mu)}^2), \label{asymptotic_equality}
\end{equation}
where $H_\mu(\xi)=O(\|\xi\|_{L^2(\mu)}^2)$ is described as follows: for $\theta'$, which is a convex combination of $\theta$ and $\theta+\xi(\theta)$ depending also on $x$,
$H_\mu(\xi)$ is defined by
\begin{equation*}
H_\mu(\xi) =-\mathbb{E}_S[yl'(-yh_\mu(x))\mathbb{E}_\mu\|\xi(\theta)\|_{\nabla_\theta^2h(\theta',x)}^2]
+ \mathbb{E}_S [l''(-yh_\mu(x))\mathbb{E}_\mu[\nabla_\theta h(\theta,x)^\top\xi(\theta)]^2].
\end{equation*}
\end{proposition}

Using this proposition, we can immediately derive a necessary local optimality condition over $\overline{\mathcal{P}}_2$ in a similar way to the finite-dimensional case.
Note that from Assumption \ref{smooth_bound_assumption}, $s_\mu(\cdot,x,y)$ is contained in $L^2(\mu)$ for any $\mu \in \mathcal{P}$.

\begin{theorem} [Necessary local optimality condition] \label{optimality_condition_thm}
Suppose Assumption \ref{smooth_bound_assumption} holds.
Let $\mu_* \in \overline{\mathcal{P}}_2$ be the local minimum of $\mathcal{L}_S(\mu)$ with respect to $d_{\mathcal{F}}$. Then we have
\begin{equation}
 \|\mathbb{E}_S[s_{\mu_*}(\theta,x,y)]\|_{L^2(\mu_*)}=0. \label{opt_necessary_condition}
\end{equation}
\end{theorem}

Next, we discuss a sufficient local optimality condition that is useful when the support of a probability measure is sufficiently small.
For a probability measure $\mu \in \mathcal{P}_2$, let us denote $M_\mu(\theta) \overset{\mathrm{def}}{=} \mathbb{E}_S[-yl'(-yh_\mu(x))\nabla_\theta^2h(\theta,x)]$
and denote ${\rm supp}^\epsilon(\mu)=\{\theta\in \Theta;\ d_2(\theta,{\rm supp}(\mu))\leq \epsilon \}$,
where $d_2$ is the Euclidean distance, i.e., ${\rm supp}^\epsilon(\mu)$ is the $\epsilon$-expansion of the support.
Noting that $\mathbb{E}_\mu[\|\xi(\theta)\|_{M_\mu(\theta')}^2]$ is the first term of $H_\mu(\xi)$, the following proposition provides the condition for the positivity of $H_\mu(\xi)$
because the remaining term is always nonnegative for a convex loss.

\begin{proposition} \label{positivity_prop}
  Suppose Assumption \ref{smooth_bound_assumption} holds.
  If there exist $\alpha >0, \epsilon>0$ for $\mu \in \mathcal{P}_2$ such that $M_\mu(\theta) \succeq \alpha I_d$ on ${\rm supp}^\epsilon(\mu)$,
then, for $\xi \in L^\infty(\mu)$ such that $\|\xi\|_{L^\infty(\mu)} \leq \epsilon$, we have $\mathbb{E}_\mu[\|\xi(\theta)\|_{M_\mu(\theta')}^2] \geq \alpha \|\xi\|_{L^2(\mu)}^2/2$. 
\end{proposition}

This proposition provides the validity of the assumption in the following sufficient local optimality theorem on $L^\infty(\mu)$ rather than $L^2(\mu)$ or $\mathcal{P}_2$, and this theorem can be shown by using Proposition \ref{asymptotic_equality_prop}.

\begin{theorem} [Sufficient local optimality condition] \label{sufficient_optimality_condition_thm}
Suppose Assumption \ref{smooth_bound_assumption} holds and $l$ is convex.
For $\mu_* \in \mathcal{P}_2$, let us assume that the condition (\ref{opt_necessary_condition}) holds and assume the existence of $\alpha >0, \gamma>0$ satisfying $H_{\mu_*}(\xi) \geq \alpha \|\xi\|_{L^2(\mu_*)}^2/2$ for $\|\xi\|_{L^\infty(\mu_*)}^2\leq \gamma$.
Then, $\xi=0$ is a local minimum of $\mathcal{L}_S((id+\xi)_\sharp \mu_*)$ on $L^\infty(\mu_*)$.
\end{theorem}

\subsection{Interior Optimality Property}
When the loss function $l$ is convex, the optimization problem (\ref{abst_prob}) is also convex with respect to $\mu \in \mathcal{P}$ in terms of affine geometry, and hence the following holds: 
for a signed Borel measure $\forall \tau$ such that $\int d\tau(\theta)=0$, 
\begin{equation*}
  \mathcal{L}_S(\mu) + \int \nabla_\mu \mathcal{L}_S(\mu)(\theta) d\tau (\theta) \leq \mathcal{L}_S(\mu + \tau),
\end{equation*}
where $\nabla_\mu \mathcal{L}_S(\mu)$ is Fr\'{e}chet derivative with respect to $\mu$: $\nabla_\mu \mathcal{L}_S(\mu)=\mathbb{E}_S[ -l'(-yh_\mu(x)) yh(\cdot, x )]$.
Thus, the equation, $\int \nabla_\mu \mathcal{L}_S(\mu)(\theta) d\tau (\theta)=0$ (for $\forall \tau$ s.t. $\int d\tau(\theta) = 0$), is the global optimality condition.
In general, this condition and the local optimality condition (\ref{opt_necessary_condition}) are different.
Indeed, when we use a Dirac measure as the initial probability measure, obtained measures by our method are also Dirac and there may exist some local minima as finite-dimensional optimization problems but the global optimality condition is not satisfied.
However, we can express the interior optimality property of local optimum by using the global optimality condition.

\begin{theorem} \label{characterization_optimality_condition_thm}
  Suppose that $h$ is a $\mathcal{C}^1$-function with respect to $\theta$ and the loss function $l$ is a $\mathcal{C}^1$-convex function.
  Let $\mu_* \in \mathcal{P}$ be a probability measure having a continuous density function.
  If ${\rm supp}(\mu_*)$ is a compact $\mathcal{C}^\infty$-manifold with boundary and $\mu_*$ satisfies the local optimality condition (\ref{opt_necessary_condition}),
  then there is neither measure $\mu$ having a continuous density such that ${\rm supp}(\mu) \subset { \rm supp(\mu_*) }$ and $\mathcal{L}_S(\mu) < \mathcal{L}_S(\mu_*)$
  nor $\mu$ not having a continuous density such that ${\rm supp}(\mu)$ is contained in the interior of $\rm supp(\mu_*)$ and $\mathcal{L}_S(\mu) < \mathcal{L}_S(\mu_*)$.
\end{theorem}

This theorem states that the optimization proceeds as long as there exists a better probability measure in support of current measure $\mu$ satisfying the same assumptions on $\mu_*$ in Theorem \ref{characterization_optimality_condition_thm} except for condition (\ref{opt_necessary_condition}).

So far, we have discussed the local optimality conditions and we have confirmed that $\|\mathbb{E}_S[s_{\mu}(\theta,x,y)]\|_{L^2(\mu)}$ for $\mu \in \mathcal{P}_2$ can be regarded as the local optimality quantity for the problem due to its continuity and the above theorems.
Therefore, the goal of an optimization method for the problems is to output a sequence $\{\mu_t\}_{t=1}^\infty\subset \mathcal{P}_2$ such that $\|\mathbb{E}_S[s_{\mu_t}(\theta,x,y)]\|_{L^2(\mu_t)}$ converges to zero.

The following proposition ensures the existence of an accumulation point of such a sequence satisfying the local optimality condition under the tightness assumption on generated probability measures.

\begin{proposition} \label{weak_convergence_prop}
  We assume a sequence $\{\mu_t\}_{t=1}^{\infty}$ in $\mathcal{P}$ is tight, that is, for arbitrary $\epsilon>0$, there exists a compact subset $\mathcal{A}\subset \Theta$ such that
  $\mu_t(\mathcal{A}) \geq 1-\epsilon$ for $\forall t \in \mathbb{N}$.
  Then, this sequence has a convergent subsequence with respect to $\|\cdot\|_{\mathcal{F}}$.
\end{proposition}

Therefore, by taking a subsequence if necessary, we can obtain a sequence $\{\mu_t\}_{t=1}^{\infty}$ to converge to a local minimum $\mu_*\in \mathcal{P}$ by an appropriate method.
Note that this convergence implies the uniform convergence $\mu_t(h(\cdot,x))\rightarrow \mu_*(h(\cdot,x))$ for all $x\in \mathcal{X}$ under Assumption \ref{continuity_assumption}.

\section{Stochastic Particle Gradient Descent} \label{sec:spgd}
In this section, we introduce a stochastic optimization method for solving problem (\ref{abst_prob}) on $\mathcal{P}_2$ and present its convergence analysis.
We first present an overview of our method again.
Let $\mu_0 \in \mathcal{P}_2$ be an initial probability measure and suppose a current probability measure $\mu$ is obtained by pushing-forward $\mu_0$ by $\phi \in L^2(\mu_0)$.
Then, $\phi$ and $\mu$ are updated along $\xi \in L^2(\mu)$ as $\phi^+ \leftarrow (id+\xi)\circ \phi$ and $\mu^+ \leftarrow (id+\xi)_\sharp \mu$.
The resulting problem is how to obtain $\xi$ to locally minimize the objective function $\mathcal{L}_S((id+\xi)_\sharp \mu)$ on $L^2(\mu)$.
We can find that by Taylor's formula (\ref{asymptotic_equality}), this objective is Fr\'{e}chet differentiable with respect to $\xi \in L^2(\mu)$ and its differential is represented by
$\mathbb{E}_S[s_\mu(\cdot,x,y)]$ via the $L^2(\mu)$-inner product.
Thus, this differential performs in function space with this inner-product like the usual gradient in a finite-dimensional space and it is expected to reduce the objective value.
We next provide a more detailed description below.

Let us denote by $B_r(\mu)$ the $r$-neighborhood of the origin in $L^2(\mu)$; $B_r(\mu)\overset{\mathrm{def}}{=}\{\xi\in L^2(\mu) \mid \|\xi\|_{L^2(\mu)}\ < r\}$.
Since the higher-order term $H_\mu(\xi) + o(\|\xi\|_{L^2(\mu)}^2)$ in (\ref{asymptotic_equality}) is $O(\|\xi\|_{L^2(\mu)}^2)$, it can be locally upper bounded by the quadratic form at $\xi = 0$.
Thus, we can assume that there exists a positive-definite smooth $(d,d)$-matrix $A_\mu(\theta)$ such that for all $\xi \in B_r(\mu)$,
\begin{equation}
\frac{1}{2}H_\mu(\xi) + o(\|\xi\|_{L^2(\mu)}^2) \leq \frac{1}{2}\mathbb{E}_\mu\|\xi\|_{A_\mu(\theta)}^2. \label{higher_order_bound}
\end{equation}
Note that we can choose scalar matrix $cI_d$ as $A_\mu$ with $c>0$ that does not depend on $\mu$ under Assumption \ref{smooth_bound_assumption}.
By Proposition \ref{asymptotic_equality_prop}, the following quadratic function with respect to $\xi$ is a local upper bound on $\mathcal{L}_S((id+\xi)_\sharp \mu)$ at $\mu \in \mathcal{P}_2$:
\begin{equation}
\mathcal{L}_S(\mu)+\mathbb{E}_\mu[\mathbb{E}_S[s_\mu(\theta,x,y)]^\top\xi(\theta) ] + \frac{1}{2}\mathbb{E}_\mu\|\xi\|_{A_\mu(\theta)}^2. \label{deterministic_subprob}
\end{equation}
Thus, minimizing (\ref{deterministic_subprob}) as a surrogate function, we can obtain $\xi \in L^2(\mu)$ to reduce the objective $\mathcal{L}_S$ and
we can make an update $\mu^+ \leftarrow (id+\xi)_\sharp \mu$ and an update $\phi^+ \leftarrow (id+\xi)\circ \phi$ for the corresponding transport map from the initial probability measure.
Practically, such a solution $\xi$ is obtained by minimizing the following stochastic approximation to (\ref{deterministic_subprob}): for randomly chosen $(x',y')$ from $S$, 
\begin{equation}
 \min_{\xi \in B_r(\mu)} \mathbb{E}_\mu[s_\mu(\theta,x',y')^\top \xi(\theta)] + \frac{1}{2}\mathbb{E}_\mu\|\xi\|_{A_\mu(\theta)}^2. \label{subprob}
\end{equation}
Note that under Assumption \ref{smooth_bound_assumption} and uniformly boundedness assumption on $A_\mu$, a positive constant $\eta_0$ exists such that for $\forall \mu \in \mathcal{P}$
and $\forall (x',y')\in \mathcal{X}\times \mathcal{Y}$,
\begin{equation*}
  \eta_0 \|A_\mu (\theta)^{-1}s_\mu(\theta,x',y') \|_{L^2(\mu)} < r. 
\end{equation*}
Thus, we can choose the step $-\eta A_\mu (\theta)^{-1}s_\mu(\theta,x',y')\ (0<\eta < \eta_0)$ as an approximated solution to (\ref{subprob})
and Lemma \ref{descent_lemma} shows the reduction of the objective function by using this step.
Moreover, if $\eta_0$ is sufficiently small, we can find this step produces a diffeomorphism (see Appendix), which preserves good properties of the initial probability
measure such as the manifold structure, and it may lead to good exploration of the proposed method by an intuition from Theorem \ref{characterization_optimality_condition_thm}. 

\begin{lemma} [Descent Lemma] \label{descent_lemma}
  Suppose Assumption \ref{smooth_bound_assumption} holds and suppose $O \prec \lambda_AI_d \preceq A_\mu(\theta) \preceq \Lambda_AI_d$.
  We set $\zeta(\theta) = -A_\mu(\theta)^{-1} s_\mu(\theta,x',y')$
Then, there exist $G>0$ and $\eta_0>0$, depending on the smoothness, the boundedness of $l,\ h,\ A_\mu$, and the radius $r$, such that for $0 < \forall \eta < \eta_0$, $\eta \zeta$ is contained in $B_r(\mu)$ and it leads to a reduction:
\begin{equation*}
\mathbb{E}_S[\mathcal{L}_S((id-\eta \zeta)_{\sharp} \mu)] 
\leq \mathcal{L}_S(\mu) - \frac{\eta}{\Lambda_A}\|\mathbb{E}_S[s_\mu(\theta,x,y)]\|_{L^2(\mu)}^2 + \eta^2G. 
\end{equation*}
\end{lemma}

This lemma means that for sufficiently small learning rates $\eta>0$, the iterate $\mu_+ \leftarrow (id+\eta\xi)_{\sharp} \mu$ strictly reduces the objective function $\mathcal{L}_S$ in the expectation when $\mu$ does not satisfy the local optimality condition (\ref{opt_necessary_condition}).
Here, we propose an algorithm called SPGD in Algorithm \ref{SPGD} to solve problem (\ref{abst_prob}) based on the above analyses. 
Note that Algorithm \ref{SPGD} is the ideal one, and hence a practical variant will be described later.

\begin{algorithm}[h]
  \caption{SPGD}
  \label{SPGD}
\begin{algorithmic}
   \STATE {\bfseries Input:} dataset $S$, initial distribution $\mu_0$, 
   the maximum number of iterations $T$, learning rates $\{\eta_k\}_{k=0}^{T-1}$\\
   $\phi_0\leftarrow id$
   \FOR{$k=0$ {\bfseries to} $T-1$}
   \STATE Randomly choose a sample $(x',y')$ from $S$ \\
   \STATE $\phi_{k+1}\leftarrow (id-\eta_ks_{\mu_k}(\cdot,x',y'))\circ \phi_k$
   \STATE $\mu_{k+1} \leftarrow (id - \eta_ks_{\mu_k}(\cdot,x',y'))_{\sharp} \mu_k$
   \ENDFOR
   \STATE Return $\phi_T$
\end{algorithmic}
\end{algorithm}

 Depending on the choice of $A_\mu$, we can derive several specific algorithms as in the traditional (stochastic) optimization literature, e.g. steepest descent method, natural gradient method, and quasi-Newton method.
Thus, Algorithm \ref{SPGD} can be regarded as the simplest form, where $A_\mu=cI_d\ (c>0)$, of SPGD.
We can obtain the convergence theorem for Algorithm \ref{SPGD} by the inequality of Lemma \ref{descent_lemma}.

\begin{theorem} [Convergence Theorem] \label{convergence_thm}
Let us make the same assumptions as in Lemma \ref{descent_lemma}.
For $\epsilon>0$, let $\eta>0$ be a constant satisfying $\eta\leq \min\{\eta_0,\frac{\epsilon}{2G}\}$.
Then an $\epsilon$-accurate solution in the expectation, i.e., $\mathbb{E}[\|\mathbb{E}_S[s_{\mu_k}(\theta,x,y)]\|_{L^2(\mu_k)}^2]\leq \epsilon$, where the outer expectation is taken with respect to the history of sample data used in learning, can be obtained at the most \\
\begin{equation}
 \frac{2(\mathcal{L}_S(\mu_0)-\inf_{\mu \in \mathcal{Q}}{\mathcal{L}}_S(\mu))}{\epsilon \eta} \label{conv_rate}
\end{equation}
iterations of Algorithm \ref{SPGD} with learning rate $\eta_k=\eta$.
\end{theorem}

Running Algorithm \ref{SPGD}, we obtain the transport map $\phi_T$.
If we choose a tractable distribution as the initial distribution $\mu_0$, we can obtain i.i.d. particles $\{\theta^0_i\}_{i=1}^M$ from $\mu_0$.
By the construction of $\phi_k$, we find that $\{\phi_k(\theta^0_i)\}_{i=1}^M$ are regarded as i.i.d. particles from the distribution $\mu_k=\phi_{k\sharp}\mu_0$.
However, note that Algorithm \ref{SPGD} is impractical because we cannot compute exact value of $h_{\mu_k}(x')$ required to get $s_{\mu_k}(\cdot,x',y')$.
Thus, we estimate it using sample average $h_{k} \sim \frac{1}{M}\sum_{i=1}^M h_{\theta^k_i}$. where $\theta^k_i = \phi_k(\theta^0_i)$.
The overall procedure is summarized in Algorithm \ref{resnet_SPGD}.
Because of the form of $id + \eta_k l'(-y'h_k)y'\nabla_\theta h(\cdot,x')$, we notice that Algorithm \ref{resnet_SPGD} iteratively stacks residual-type layers \cite{he2016deep}
and so a residual network to output an infinite ensemble is built naturally.
We can also derive more practical variant Algorithm \ref{practical_SPGD} without resampling in Algorithm \ref{resnet_SPGD}, that is, using the same seeds $\{\theta^0_i\}_{i=1}^M$ over all iterations.

\begin{algorithm}[h]
  \caption{SPGD - building residual network -}
  \label{resnet_SPGD}
\begin{algorithmic}
   \STATE {\bfseries Input:} dataset $S$, initial distribution $\mu_0$, 
   the maximum number of iterations $T$, the number of particles $M$, learning rates $\{\eta_k\}_{k=0}^{T-1}$\\
   $\phi_0\leftarrow id$   
   \FOR{$k=0$ {\bfseries to} $T-1$}
   \STATE Independently draw particles $\{\theta^0_i\}_{i=1}^M$ from $\mu_0$\\
   \STATE $\{\theta^k_i\}_{i=1}^M \leftarrow \{\phi_k(\theta^0_i)\}_{i=1}^M$
   \STATE Randomly choose a sample $(x',y')$ from $S$ \\
   \STATE $h_k \leftarrow \frac{1}{M}\sum_{i=1}^M h_{\theta^k_i}(x')$
   \STATE $\phi_{k+1}\leftarrow ( id + \eta_k l'(-y'h_k)y'\nabla_\theta h(\cdot,x'))\circ \phi_k$
   \ENDFOR
   \STATE Return $\{\theta^{T}_i\}_{i=1}^M$
\end{algorithmic}
\end{algorithm}

\begin{algorithm}[h]
  \caption{SPGD - practical variant -}
  \label{practical_SPGD}
\begin{algorithmic}
   \STATE {\bfseries Input:} dataset $S$, initial distribution $\mu_0$, 
   the maximum number of iterations $T$, the number of particles $M$, learning rates $\{\eta_k\}_{k=0}^{T-1}$\\
   Independently draw particles $\{\theta^0_i\}_{i=1}^M$ from $\mu_0$\\
   \FOR{$k=0$ {\bfseries to} $T-1$}
   \STATE Randomly choose a sample $(x',y')$ from $S$ \\
   \STATE $h_k \leftarrow \frac{1}{M}\sum_{i=1}^M h_{\theta^k_i}(x')$
   \STATE $\{\theta^{k+1}_i\}_{i=1}^M\leftarrow \{\theta^{k}_i + \eta_k l'(-y'h_k)y'\nabla_\theta h(\theta^k_i,x')\}_{i=1}^M$
   \ENDFOR
   \STATE Return $\{\theta^{T}_i\}_{i=1}^M$
\end{algorithmic}
\end{algorithm}

We next describe a perspective of SPGD as an extension of vanilla SGD; moreover, the other perspectives are provided in the Appendix, which certainly leads to a deeper understanding of the method. 

\subsection{Extension of Vanilla Stochastic Gradient Descent}
If we adopt the sum of Dirac measures as the initial distribution $\mu_0$, then Algorithm \ref{SPGD} and \ref{practical_SPGD} become the same method by initializing particles $\{\theta_i^0\}_{i=1}^M$ to be the support of $\mu_0$.
Moreover, we can see that the step of Algorithm \ref{practical_SPGD} is the same as that of vanilla SGD for the nonweighted voting problem:
$\min_{\{\theta_i\} \in \Theta^M}\mathbb{E}_S[ l(-\frac{1}{M}\sum_{i=1}^M yh(\theta_i,x))]$.
Specifically, we can say that the vanilla SGD for learning a base classifier is the method to optimize a Dirac measure and is none other than Algorithm \ref{SPGD} with a Dirac measure $\mu_0$.
In other words, Algorithm \ref{SPGD} is an extension of the vanilla SGD to the method for optimizing a general probability measure.

From this viewpoint of Algorithm \ref{practical_SPGD}, we can introduce some existing techniques and extensions to our method.
For instance, we can use accelerating techniques such as Nesterov's momentum method \cite{nes2004}, which is also used in our experiments to accelerate the convergence.

Moreover, we can extend Algorithm \ref{practical_SPGD} to the multiclass classification problems. 
Let us consider the $c$-classes classification problem.
We denote the binary vector for the class by ${\bf y}$, that is, for the $i$-th class, only the $i$-th element $y_i$ is one and the other elements are zeros.
The output of the classifier ${\bf h}$ is extended to the range $[0,1]^c$, which represent the confidences of each class such as the softmax function.
Then, the SGD for the problem $\min_{\{\theta_i\} \in \Theta^M}\mathbb{E}_S[ l(-\frac{1}{M}\sum_{i=1}^M {\bf y}^\top{\bf h}(\theta_i,x))]$ is the extension of Algorithm \ref{practical_SPGD}
to the multiclass problem.

\section{Numerical Experiments} \label{sec:experiments}

\renewcommand{\arraystretch}{1}
\begin{table*}[t]
\caption{Test classification accuracy on binary and multiclass classification.}
\label{comparison_table}
\vskip 0.15in
\begin{center}
\begin{scriptsize}
\begin{sc}
\begin{tabular}{c|cc|cc|cc}
Dataset & 
LogReg &
SPGD(logreg) &
MLP(exp) &
SPGD(exp) &
MLP(log) &
SPGD(log)\\
\hline
\multirow{2}{*}{breastcancer} & 
{\bf 0.966} &
0.965 &
0.965 &
{\bf 0.971} &
0.968 &
{\bf 0.971} \\
 &
{\bf (0.0187)} &
(0.0177) &
(0.0210) &
{\bf (0.0174)} &
(0.0110) &
{\bf (0.0174)} \\
\multirow{2}{*}{diabetes} \rule[0mm]{0mm}{3mm}&
0.755 &
{\bf 0.761} &
{\bf 0.764} &
0.756 &
0.738 &
{\bf 0.757} \\
 & 
(0.0464) &
{\bf (0.0435)} &
{\bf (0.0366)} &
(0.0447) &
(0.0524) &
{\bf (0.0400)} \\
\multirow{2}{*}{german} \rule[0mm]{0mm}{3mm}& 
{\bf 0.769} &
0.763 &
0.738 &
{\bf 0.769} &
0.724 &
{\bf 0.775} \\
 &
{\bf (0.0406)} &
(0.0390) &
(0.0178) &
{\bf (0.0381)} &
(0.0393) &
{\bf (0.0356)} \\
\multirow{2}{*}{ionosphereo} \rule[0mm]{0mm}{3mm}& 
{\bf 0.892} &
0.886 &
0.914 &
{\bf 0.937} &
0.923 &
{\bf 0.937} \\
&
{\bf (0.0400)} &
(0.0383) &
(0.0512) &
{\bf (0.0274)} &
(0.0339) &
{\bf (0.0274)} \\
\multirow{2}{*}{glass} \rule[0mm]{0mm}{3mm}& 
0.566 &
{\bf 0.622} &
0.477 &
{\bf 0.616} &
0.619 &
{\bf 0.659} \\
&
(0.0655) &
{\bf (0.0692)} &
(0.1127) &
{\bf (0.0595)} &
(0.1144) &
{\bf (0.1033)} \\
\multirow{2}{*}{segment} \rule[0mm]{0mm}{3mm}& 
{\bf 0.934} &
0.913 &
0.717 &
{\bf 0.953} &
0.961 &
{\bf 0.970} \\
&
{\bf (0.0148)} &
(0.0143) &
(0.1104) &
{\bf (0.0100)} &
(0.0082) &
{\bf (0.0089)} \\
\multirow{2}{*}{vehicle} \rule[0mm]{0mm}{3mm}& 
0.771 &
{\bf 0.780} &
0.759 &
{\bf 0.838} &
0.794 &
{\bf 0.829} \\
&
(0.0422) &
(0.0248) &
(0.0372) &
{\bf (0.0451)} &
(0.0525) &
{\bf (0.0370)} \\
\multirow{2}{*}{wine} \rule[0mm]{0mm}{3mm}& 
0.968 &
{\bf 0.978} &
0.949 &
{\bf 0.974} &
0.963 &
{\bf 0.984} \\
&
(0.0321) &
(0.0377) &
(0.0519) &
{\bf (0.0414)} &
(0.0552) &
{\bf (0.0246)} \\
\multirow{2}{*}{covertype} \rule[0mm]{0mm}{3mm}& 
0.720 &
{\bf 0.738} &
{\bf 0.772} &
0.763 &
0.772 &
{\bf 0.806} \\
&
(0.0071) &
{\bf (0.0056)} &
{\bf (0.0269) } &
(0.0255) &
(0.0271) &
{\bf (0.0247)} \\
\hline
\end{tabular}
\end{sc}
\end{scriptsize}
\end{center}
\vskip -0.1in
\end{table*}
\renewcommand{\arraystretch}{1}

\subsection{Synthetic Data}
We first present how our method behaves by using toy data: two-dimensional double circle data.
We ran Algorithm \ref{practical_SPGD} for Example \ref{ex_linear_classifier} of a binary linear model with exponential loss; we solved the following:
\[ \min_{\theta_i \in \mathbb{R}^2, b_i\in \mathbb{R}}\mathbb{E}_S\left[ \exp\left(-\frac{1}{M}\sum_{i=1}^M Y\tanh(\theta_i^\top X + b_i)\right)\right]. \]
The number of particles was set to be $20$.
The behavior of the method is shown in Figure \ref{toy_experiments}.
The upper-left part shows weights $\theta$ of the initial particles and the upper-right part shows $\theta$ of the final particles.
The bottom row represents predicted labels by the initial particles (left) and the final particles (right).
It can be seen that the data are well classified by the locations of the particles using this method.

\subsection{Real Data}
Next, we present the results of experiments on binary and multiclass classification tasks in a real dataset.
We ran Algorithm \ref{practical_SPGD} with momentum for logistic regression and three-layer perceptrons 
where we set the number of hidden units to be the same as the input dimension and we used sigmoid activation for the output of the hidden layer.
For the last layer of multilayer perceptrons, we used softmax output with the exponential loss or the logarithmic loss function.
The number of particles was set to be $10$ or $30$.
Each element of initial particles was sampled from the normal distribution $\mu_0$ with zero mean and standard deviation of $0.01$ to bias parameters and of $1$ to weight parameters.
To evaluate the performance of the SPGD, we also ran logistic regression and multilayer perceptron, whose structure is the same as used for SPGD.

We used the UCI datasets: breast-cancer, diabetes, german, and ionosphere for binary classification; glass, segment, vehicle, wine, and covertype for multiclass classification.
We used the following experimental procedure as in \cite{cms2014}; we first divided each dataset into 10 folds. 
For each run $i\in \{1,\ldots,10\}$, we used fold $i$ for validation, used fold $i+1\ ({\rm mod}\ 10)$ for testing, and used the other folds for training.
We performed each method on the training dataset with several hyper-parameter settings and we chose the best parameter on the validation dataset.
Finally, we evaluated it on the testing dataset.

The mean classification accuracy and the standard deviation are presented in Table \ref{comparison_table}.
Notations SPGD(LOGREG), SPGD(EXP), and SPGD(LOG) stand for SPGD for logistic regression, multilayer perceptrons with exponential loss, and with logarithmic loss function, respectively.
Although SPGD did not improve logistic regression on some datasets, it showed overall improvements over base models on the other settings.
Thus, we confirmed the effectiveness of our method. 

\begin{figure}[H]
\begin{center}
\begin{tabular}{cccc}
  \resizebox{40mm}{!}{ \includegraphics[angle=0]{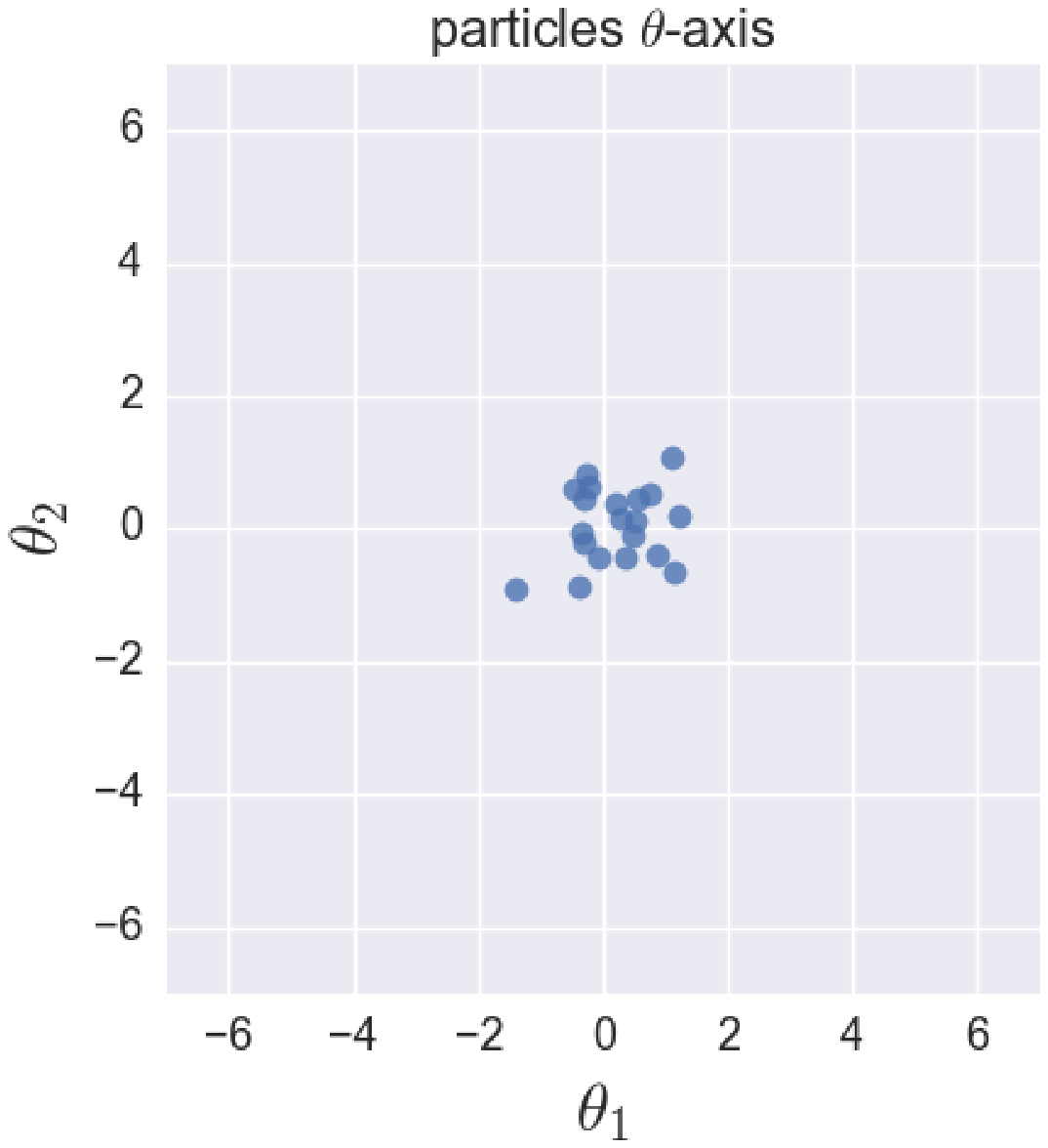} } 
\hspace{-8mm} \resizebox{40mm}{!}{ \includegraphics[angle=0]{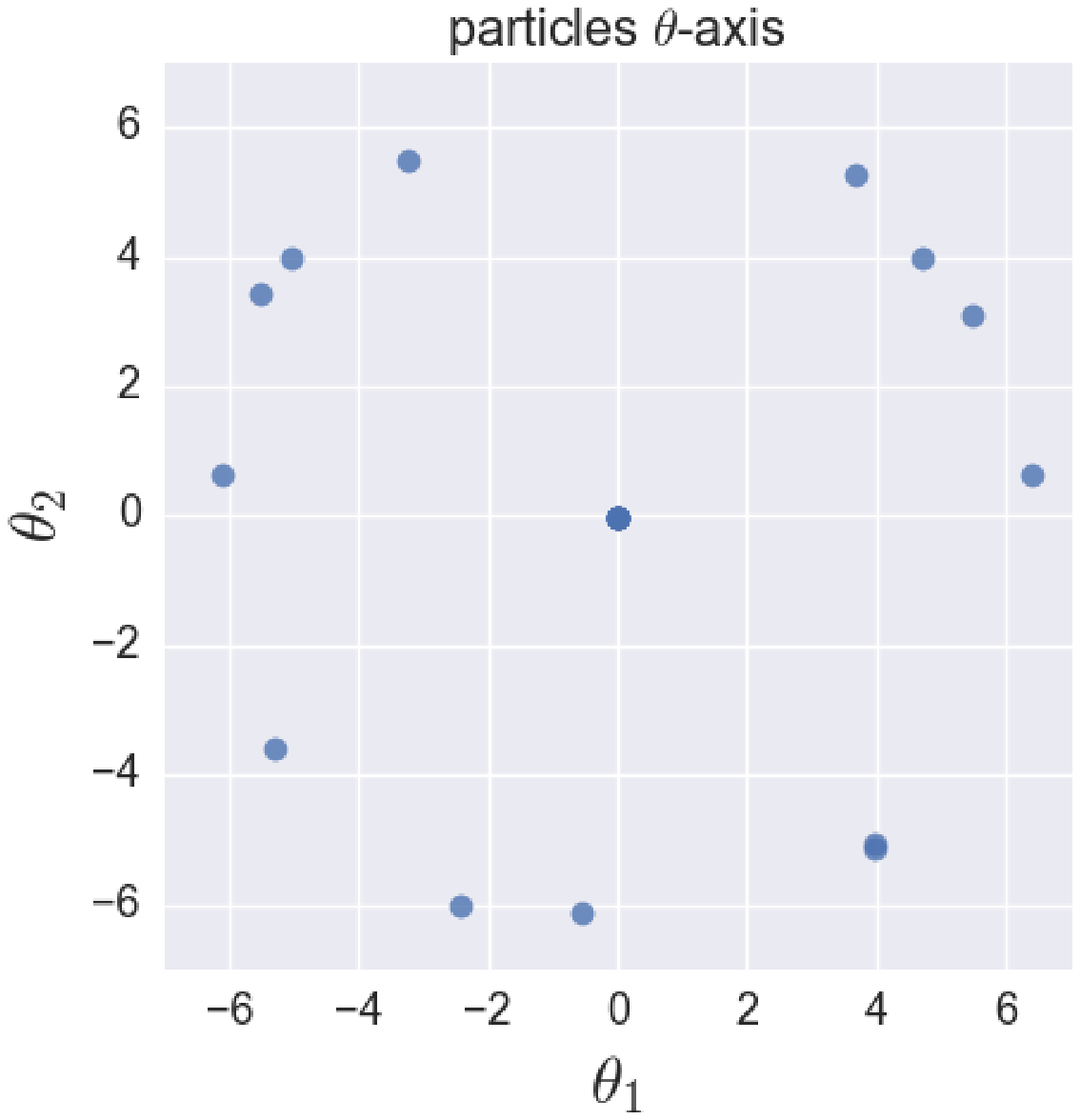} } 
  \resizebox{40mm}{!}{ \includegraphics[angle=0]{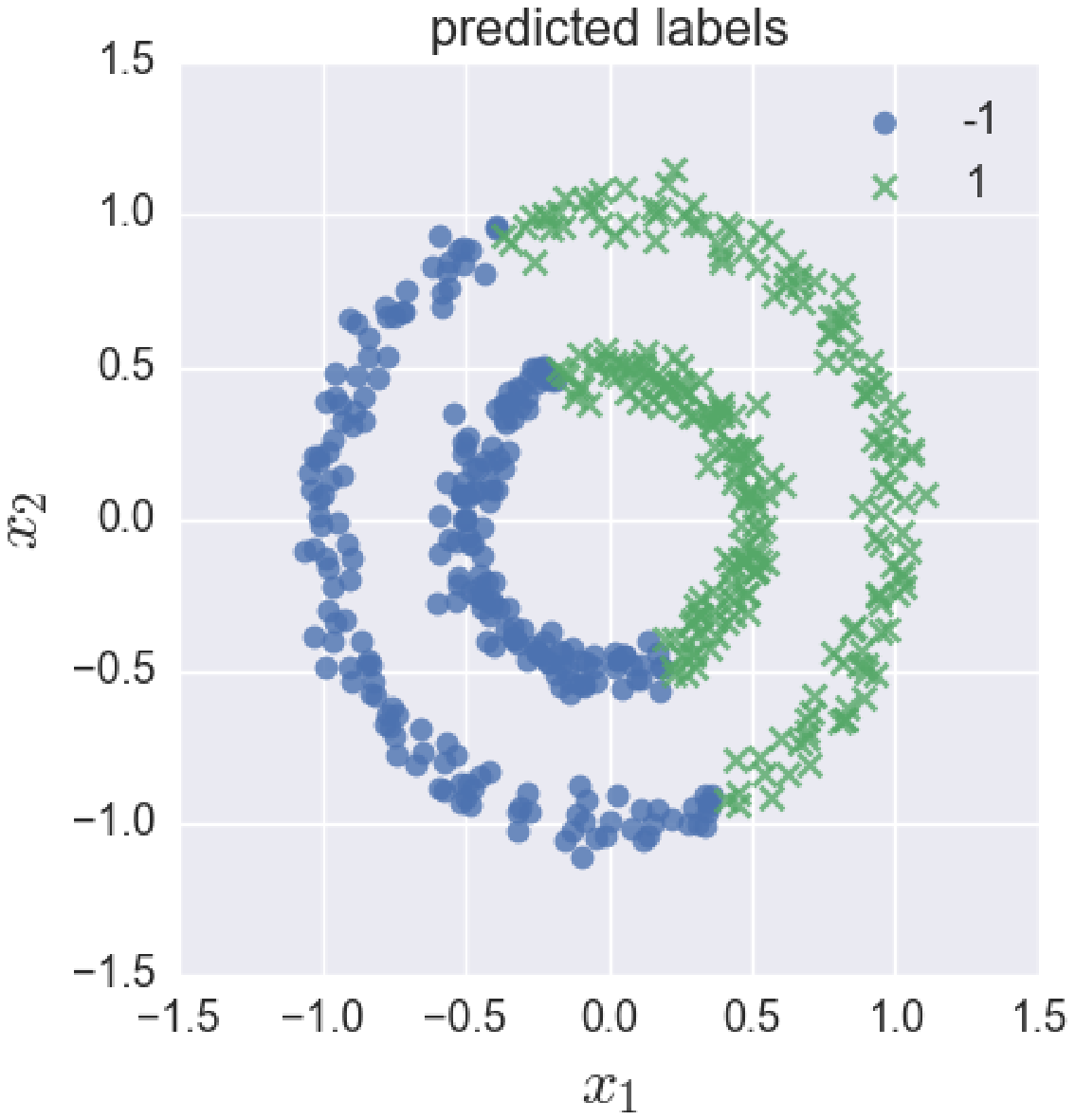} }
\hspace{-8mm} \resizebox{40mm}{!}{ \includegraphics[angle=0]{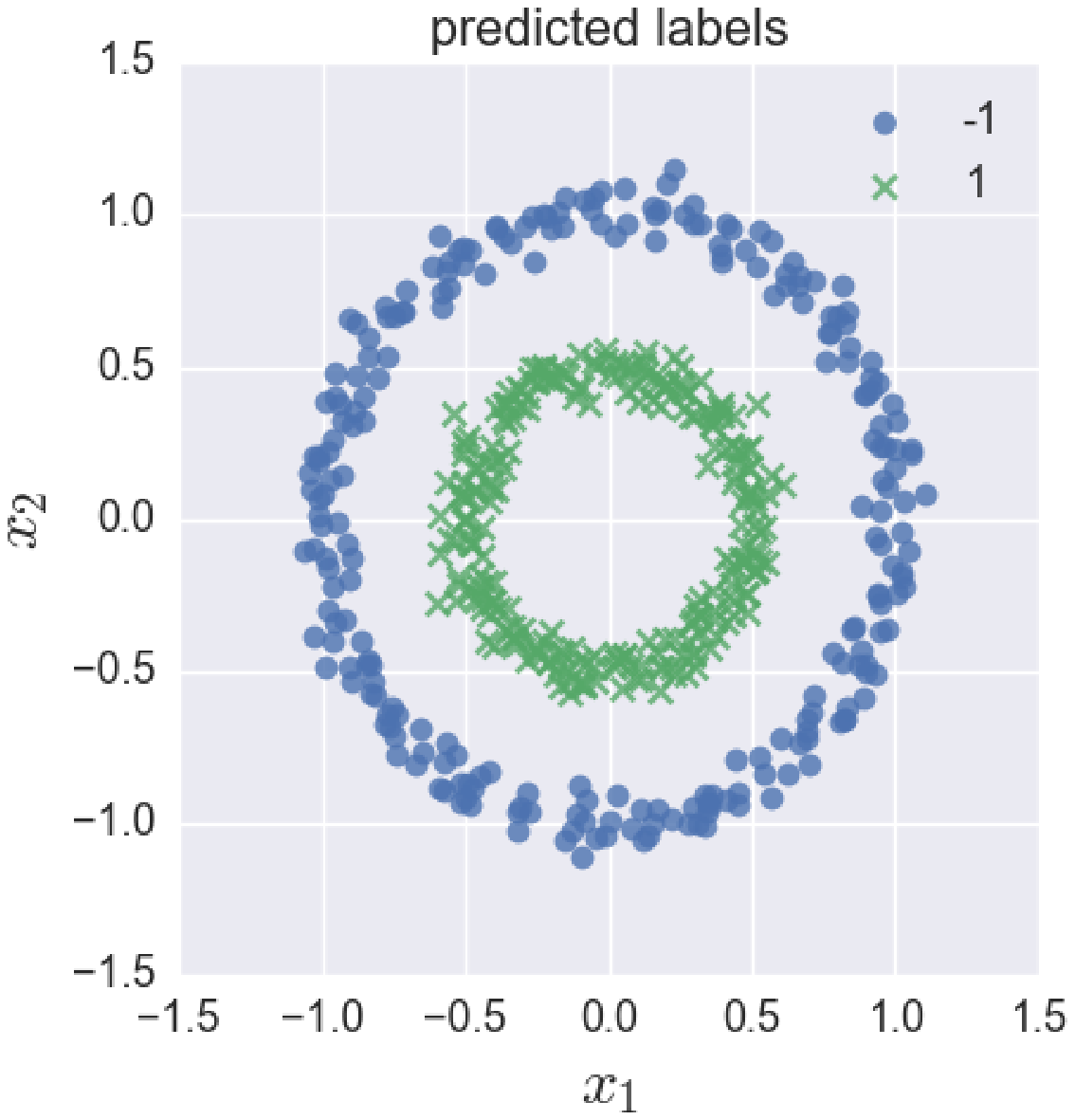} } \\
\end{tabular}
\caption{ Toy example of the SPGD method (upper-left: weights of the initial particles; upper-right: weights of the final particles;
  bottom-left: predicted labels by the initial particles; bottom-right: predicted labels by the final particles). }
\label{toy_experiments}
\end{center}
\end{figure}

\section{Conclusion}
We introduced the infinite ensemble and derived the generalization error bound by extending the well-known results for the convex combination.
We also explored an optimality condition for the empirical risk minimization problem for the infinite ensemble learning.
To solve this problem, we proposed a stochastic optimization method with the convergence analysis.

\bibliographystyle{plain}

\begin{thebibliography}{10}

\bibitem{ags2008}
Luigi Ambrosio, Nicola Gigli, and Giuseppe Savar{\'e}.
\newblock {\em Gradient Flows in Metric Spaces and in the Space of Probability
  Measures}.
\newblock Lectures in Mathematics. ETH Z{\"u}rich. Birkh{\"a}user Basel, 2008.

\bibitem{bach2014breaking}
Francis Bach.
\newblock Breaking the curse of dimensionality with convex neural networks.
\newblock {\em Journal of Machine Learning Research}, 18(19):1--53, 2014.

\bibitem{bk1999}
Eric Bauer and Ron Kohavi.
\newblock An empirical comparison of voting classification algorithms: Bagging,
  boosting and variants.
\newblock {\em Machine Learning}, 36:105--139, 1999.

\bibitem{bjm2015}
Martin Bauer, Sarang Joshi, and Klas Modin.
\newblock Diffeomorphic density matching by optimal information transport.
\newblock {\em SIAM Journal on Imaging Sciences}, 8(3):1718--1751, 2015.

\bibitem{bengio2006convex}
Yoshua Bengio, Nicolas~L Roux, Pascal Vincent, Olivier Delalleau, and Patrice
  Marcotte.
\newblock Convex neural networks.
\newblock In {\em Advances in Neural Information Processing Systems 19}, pages
  123--130, 2006.

\bibitem{bre1999}
Leo Breiman.
\newblock Prediction games and arcing algorithms.
\newblock {\em Neural Computation}, 11:1493--1517, 1999.

\bibitem{1711.10927}
Changyou Chen and Ruiyi Zhang.
\newblock Particle optimization in stochastic gradient mcmc, 2017.

\bibitem{cms2014}
Corinna Cortes, Mehryar Mohri, and Umar Syed.
\newblock Deep boosting.
\newblock In {\em Proceedings of the 31st International Conference on Machine
  Learning}, pages 1179--1187, 2014.

\bibitem{dai2016provable}
Bo~Dai, Niao He, Hanjun Dai, and Le~Song.
\newblock Provable {B}ayesian inference via particle mirror descent.
\newblock In {\em Proceedings of the 19th International Conference on
  Artificial Intelligence and Statistics}, pages 985--994, 2016.

\bibitem{dai2014scalable}
Bo~Dai, Bo~Xie, Niao He, Yingyu Liang, Anant Raj, Maria-Florina~F Balcan, and
  Le~Song.
\newblock Scalable kernel methods via doubly stochastic gradients.
\newblock In {\em Advances in Neural Information Processing Systems 27}, pages
  3041--3049, 2014.

\bibitem{pd2009}
Philippe Delano{\"e}.
\newblock Differential geometric heuristics for {R}iemannian optimal mass
  transportation.
\newblock In {\em Differential Equations - Geometry, Symmetries and
  Integrability}, pages 49--73, 2009.

\bibitem{die2000}
Thomas~G Dietterich.
\newblock An experimental comparison of three methods for constructing
  ensembles of decision trees: Bagging, boosting and randomization.
\newblock {\em Machine Learning}, 40:139--157, 2000.

\bibitem{dieuleveut2016nonparametric}
Aymeric Dieuleveut and Francis Bach.
\newblock Nonparametric stochastic approximation with large step-sizes.
\newblock {\em The Annals of Statistics}, 44(4):1363--1399, 2016.

\bibitem{dud1968}
Richard~M Dudley.
\newblock Distances of probability measures and random variables.
\newblock {\em The Annals of Mathematical Statistics}, 39(5):1563--1572, 1968.

\bibitem{dud1999}
Richard~M Dudley.
\newblock {\em Uniform Central Limit Theorems}.
\newblock Cambridge University Press, 1999.

\bibitem{freund1997decision}
Yoav Freund and Robert~E Schapire.
\newblock A decision-theoretic generalization of on-line learning and an
  application to boosting.
\newblock {\em Journal of Computer and System Sciences}, 55(1):119--139, 1997.

\bibitem{friedman2001greedy}
Jerome~H Friedman.
\newblock Greedy function approximation: A gradient boosting machine.
\newblock {\em The Annals of Statistics}, 29(5):1189--1232, 2001.

\bibitem{friedman2000additive}
Jerome~H Friedman, Trevor Hastie, and Robert Tibshirani.
\newblock Additive logistic regression: A statistical view of boosting.
\newblock {\em The Annals of Statistics}, 28(2):337--407, 2000.

\bibitem{gp1974}
Victor Guillemin and Alan Pollack.
\newblock {\em Differential Topology}.
\newblock Prentice-Hall, 1974.

\bibitem{he2016deep}
Kaiming He, Xiangyu Zhang, Shaoqing Ren, and Jian Sun.
\newblock Deep residual learning for image recognition.
\newblock In {\em Proceedings of the IEEE Conference on Computer Vision and
  Pattern Recognition}, pages 770--778, 2016.

\bibitem{hormander1963}
Lars H{\"o}rmander.
\newblock {\em Linear Partial Differential Operators}.
\newblock Springer, 1963.

\bibitem{kivinen2004online}
Jyrki Kivinen, Alexander~J Smola, and Robert~C Williamson.
\newblock Online learning with kernels.
\newblock {\em IEEE Transactions on Signal Processing}, 52(8):2165--2176, 2004.

\bibitem{kpl2003}
Vladimir Koltchinskii, Dmitriy Panchenko, and Fernando Lozano.
\newblock Bounding the generalization error of convex combinations of
  classifiers: Balancing the dimensionality and the margins.
\newblock {\em Annals of Applied Probability}, 13(1):213--252, 2003.

\bibitem{kp2002}
Vladimir Koltchinskii and Dmitry Panchenko.
\newblock Empirical margin distributions and bounding the generalization error
  of combined classifiers.
\newblock {\em The Annals of Statistics}, 30(1):1--50, 2002.

\bibitem{kp2005}
Vladimir Koltchinskii and Dmitry Panchenko.
\newblock Complexities of convex combinations and bounding the generalization
  error in classification.
\newblock {\em The Annals of Statistics}, 33(4):1455--1496, 2005.

\bibitem{le2007continuous}
Nicolas Le~Roux and Yoshua Bengio.
\newblock Continuous neural networks.
\newblock In {\em Proceedings of 11th International Conference on Artificial
  Intelligence and Statistics}, pages 404--411, 2007.

\bibitem{liu2017stein}
Qiang Liu.
\newblock {S}tein variational gradient descent as gradient flow.
\newblock In {\em Advances in Neural Information Processing Systems 30}, pages
  3117--3125, 2017.

\bibitem{lw2016}
Qiang Liu and Dilin Wang.
\newblock {S}tein variational gradient descent: A general purpose {B}ayesian
  inference algorithm.
\newblock In {\em Advances in Neural Information Processing Systems 29}, pages
  2378--2386. 2016.

\bibitem{livni2017learning}
Roi Livni, Daniel Carmon, and Amir Globerson.
\newblock Learning infinite layer networks without the kernel trick.
\newblock In {\em Proceedings of the 34th International Conference on Machine
  Learning}, pages 2198--2207, 2017.

\bibitem{luenberger1969optimization}
David~G. Luenberger.
\newblock {\em Optimization by Vector Space Methods}.
\newblock John Wiley \& Sons, 1969.

\bibitem{mackay1992practical}
David~JC MacKay.
\newblock A practical {B}ayesian framework for backpropagation networks.
\newblock {\em Neural Computation}, 4(3):448--472, 1992.

\bibitem{mackay1995probable}
David~JC MacKay.
\newblock Probable networks and plausible predictions—a review of practical
  {B}ayesian methods for supervised neural networks.
\newblock {\em Network: Computation in Neural Systems}, 6(3):469--505, 1995.

\bibitem{mbbf2000}
Llew Mason, Jonathan Baxter, Peter~L Bartlett, and Marcus Frean.
\newblock Functional gradient techniques for combining hypotheses.
\newblock In {\em Advances in Large Margin Classifiers}. MIT Press, 2000.

\bibitem{CIS-137913}
Alfred M{\"u}ller.
\newblock Integral probability metrics and their generating classes of
  functions.
\newblock {\em Advances in Applied Probability}, 29:429--443, 1997.

\bibitem{neal2012bayesian}
Radford~M Neal.
\newblock {\em {B}ayesian Learning for Neural Networks}, volume 118.
\newblock Springer Science \& Business Media, 2012.

\bibitem{nes2004}
Yurii Nesterov.
\newblock {\em Introductory Lectures on Convex Optimization: A Basic Course}.
\newblock Kluwer Academic Publishers, 2004.

\bibitem{nm1994}
Whitney~K Newey and Daniel McFadden.
\newblock {\em Large Sample Estimation and Hypothesis Testing}, volume~4.
\newblock 1994.

\bibitem{otto2001}
Felix Otto.
\newblock The geometry of dissipative evolution equations: The porous medium
  equation.
\newblock {\em Communications in Partial Differential Equations},
  26(1-2):101--174, 2001.

\bibitem{rahimi2008random}
Ali Rahimi and Benjamin Recht.
\newblock Random features for large-scale kernel machines.
\newblock In {\em Advances in Neural Information Processing Systems 21}, pages
  1177--1184, 2008.

\bibitem{rahimi2009weighted}
Ali Rahimi and Benjamin Recht.
\newblock Weighted sums of random kitchen sinks: Replacing minimization with
  randomization in learning.
\newblock In {\em Advances in Neural Information Processing Systems 22}, pages
  1313--1320, 2009.

\bibitem{rw2005}
Gunnar R{\"a}tsch and Manfred~K Warmuth.
\newblock Efficient margin maximizing with boosting.
\newblock {\em Journal of Machine Learning Research}, 6(Dec):2131--2152, 2005.

\bibitem{rezende2015variational}
Danilo Rezende and Shakir Mohamed.
\newblock Variational inference with normalizing flows.
\newblock In {\em Proceedings of the 32nd International Conference on Machine
  Learning}, pages 1530--1538, 2015.

\bibitem{rosset2007l1}
Saharon Rosset, Grzegorz Swirszcz, Nathan Srebro, and Ji~Zhu.
\newblock $l_1$-regularization in infinite dimensional feature spaces.
\newblock {\em Lecture Notes in Computer Science}, 4539:544, 2007.

\bibitem{rzh2004}
Saharon Rosset, Ji~Zhu, and Trevor Hastie.
\newblock Boosting as a regularized path to a maximum margin classifier.
\newblock {\em Journal of Machine Learning Research}, 5:941--973, 2004.

\bibitem{rsd2004}
Cynthia Rudin, Robert~E Schapire, and Ingrid Daubechies.
\newblock Boosting based on a smooth margin.
\newblock In {\em Proceedings of the Annual Conference on Learning Theory},
  pages 502--517, 2004.

\bibitem{ruiz2012}
David Ruiz.
\newblock A note on the uniformity of the constant in the {P}oincar{\'e}
  inequality.
\newblock {\em Advanced Nonlinear Studies}, 12:889--903, 2012.

\bibitem{sf2012}
Robert~E Schapire and Yoav Freund.
\newblock {\em Boosting: Foundations and algorithms}.
\newblock MIT press, 2012.

\bibitem{sfbl1998}
Robert~E Schapire, Yoav Freund, Peter Bartlett, and Wee~Sun Lee.
\newblock Boosting the margin: A new explanation for the effectiveness of
  voting methods.
\newblock {\em The Annals of Statistics}, 26(5):1651--1686, 1998.

\bibitem{vw1996}
AW~van~der Vaart and Jon Wellner.
\newblock {\em Weak Convergence and Empirical Processes: With Applications to
  Statistics}.
\newblock Springer, 1996.

\bibitem{villani2009}
C{\'e}dric Villani.
\newblock {\em Optimal Transport: Old and New}.
\newblock Springer-Verlag Berlin Heidelberg, 2009.

\bibitem{vj2001}
Paul Viola and Michael Jones.
\newblock Rapid object detection using a boosted cascade of simple features.
\newblock In {\em Proceedings of IEEE Computer Society Conference on Computer
  Vision and Pattern Recognition}, pages 511--518, 2001.

\bibitem{welling2011bayesian}
Max Welling and Yee~W Teh.
\newblock {B}ayesian learning via stochastic gradient langevin dynamics.
\newblock In {\em Proceedings of the 28th International Conference on Machine
  Learning}, pages 681--688, 2011.

\bibitem{1512.02394}
Changbo Zhu and Huan Xu.
\newblock Online gradient descent in function space, 2015.

\end{thebibliography}

\clearpage
\renewcommand{\thesection}{\Alph{section}}
\renewcommand{\thesubsection}{\Alph{section}. \arabic{subsection}}
\renewcommand{\thetheorem}{\Alph{theorem}}
\renewcommand{\thelemma}{\Alph{lemma}}
\renewcommand{\theproposition}{\Alph{proposition}}
\renewcommand{\thedefinition}{\Alph{definition}}
\renewcommand{\thecorollary}{\Alph{corollary}}
\renewcommand{\theassumption}{\Alph{assumption}}
\renewcommand{\theexample}{\Alph{example}}

\setcounter{section}{0}
\setcounter{subsection}{0}
\setcounter{theorem}{0}
\setcounter{lemma}{0}
\setcounter{proposition}{0}
\setcounter{definition}{0}
\setcounter{corollary}{0}
\setcounter{assumption}{0}
\setcounter{example}{0}

\part*{\Large{Appendix}}

\section{Generalization Bounds}
In this section, we give the proof of generalization bounds of majority vote classifiers.
\begin{proof}[ Proof of Theorem \ref{generalization_thm_1}. ]
For a function class $\mathcal{J}$ and a dataset $S=\{x_i\}_{i=1}^N$, we denote empirical Rademacher complexity by $\hat{\mathcal{R}}_S(\mathcal{J})$ and denote Rademacher complexity by $\mathcal{R}_N(\mathcal{J})$; let $\sigma =(\sigma_i)_{i=1}^N$ be i.i.d random variables taking $-1$ or $1$ with equal probability and let $S$ be distributed according to $D^N$,
\[ \hat{\mathcal{R}}_S(\mathcal{J}) = \mathbb{E}_\sigma\left[ \sup_{f\in\mathcal{J}}\frac{1}{N}\sum_{i=1}^N\sigma_if(x_i)\right],\ \  \mathcal{R}_N(\mathcal{J})=\mathbb{E}_{D^N}[\hat{\mathcal{R}}_S(\mathcal{J})]. \]

The following lemma indicates averaging operator by probability measure dose not increase Rademacher complexity, which is a counterpart of it for convex combinations \cite{kp2002}.
\begin{lemma} 
The following inequality is valid for an arbitrary data set $S$.
\[ \hat{\mathcal{R}}_S(\mathcal{G}) = \hat{\mathcal{R}}_S(\mathcal{H}). \]
\end{lemma}
\begin{proof}
The proof is concluded by
\begin{align*}
\hat{\mathcal{R}}_S(\mathcal{G}) &= \mathbb{E}_\sigma \left[ \sup_{f\in \mathcal{G}} \frac{1}{N}\sum_{i=1}^{N}\sigma_if(x_i)\right] \\
&= \mathbb{E}_\sigma \left[ \sup_{\mu \in \mathcal{P} }\frac{1}{N} \sum_{i=1}^{N} \sigma_i \mathbb{E}_\mu[h(\theta,x_i)] \right]  \\
&= \mathbb{E}_\sigma \left[ \sup_{\mu \in \mathcal{P} }\frac{1}{N} \mathbb{E}_\mu \left[ \sum_{i=1}^{N} \sigma_ih(\theta,x_i) \right] \right]  \\
&= \mathbb{E}_\sigma \left[ \sup_{\theta \in \Theta} \frac{1}{N} \sum_{i=1}^{N}\sigma_ih(\theta,x_i)\right] = \hat{\mathcal{R}}_S(\mathcal{H}).  
\end{align*}
\end{proof}

Using this lemma, we can obtain the following theorem in the same manner as in \cite{kp2002}.

\begin{theorem}
  Let $N\in \mathbb{N}$ be the number of data.
  Then, for $\forall \delta \in (0,1)$ with probability at least $1-\delta$ over the random choice of $S$ for $\forall h_\mu \in \mathcal{G}$ we have 
\[ \mathbb{P}_D[yh_\mu(x)\leq 0] \leq \inf_{\rho \in (0,1]} \left( \mathbb{P}_S[yh_\mu(x)\leq \rho] 
+ \frac{8}{\rho}\mathcal{R}_N(\mathcal{H}) + \sqrt{\frac{\log \log_2(2/\rho)}{N}} \right) +\sqrt{\frac{\log(2/\delta)}{2N}}. \]
\end{theorem}

Combining this proposition and the following Rademacher processing variant of Dudley integral bound \cite{dud1999} under Assumption \ref{covering_numbers_assumption}, we can finish the proof of Theorem \ref{generalization_thm_1}.

\begin{theorem} [\cite{dud1999}]
There is a constant $K>0$ such that for every data $S=\{x_i\}_{i=1}^N$,
\[ \frac{1}{\sqrt{N}} \mathbb{E}_\sigma \sup_{f\in\mathcal{J}}\left[\sum_{i=1}^N \sigma_if(x_i)\right] \leq K \int_0^\delta \log^{1/2}N_{m_N}(\mathcal{J}, \epsilon)d\epsilon, \]
where $\delta = \sup_{f\in \mathcal{J}}\sqrt{\frac{1}{N}\sum_{i=1}^Nf^2(x_i)}$ and $m_N$ is the empirical measure supported on the given sample $\{x_1,\ldots,x_N\}$.
\end{theorem}
\end{proof}

\begin{proof}[ Proof of Theorem \ref{generalization_thm_2}. ]
We first give Proposition \ref{mv_and_conv_relation_prop} to prove Theorem \ref{generalization_thm_2}.
Let ${\rm conv}(\mathcal{H})$ denote the set of all convex combinations of base classifiers in $\mathcal{H}$.
Proposition \ref{mv_and_conv_relation_prop} gives the relation between covering numbers of the set of convex combinations and the set of infinite ensembles.

\begin{proposition} \label{mv_and_conv_relation_prop}
If the feature space $\mathcal{X}$ is compact, then $\forall \epsilon > 0, \forall r>1$ and the Borel probability measure $\forall m$ on $\mathcal{X}$, we have $N_m(r\epsilon, \mathcal{G}) \leq N_m(\epsilon,{\rm conv}(H))$.
\end{proposition}

\begin{proof}[ Proof of Proposition \ref{mv_and_conv_relation_prop}. ]
Let $\mu$ be a probability measure on $\Theta$.
Since $\mathcal{X}$ is compact, $h(\theta,x)$ is uniformly bounded, measurable w.r.t. $\theta$, and continuous w.r.t. $x$, the condition of uniform law of large numbers (see Lemma 2.4 in \cite{nm1994}) is satisfied.
Specifically, for arbitrary $\epsilon>0$, we can draw particles $\{\theta_j\}_{j=1}^s$ according to $\mu$ satisfying $\sup_{x\in\mathcal{X}}|\frac{1}{s}\sum_{j=1}^s h(\theta_j,x)-h_\mu(x)| \leq \epsilon$.
This uniform bound implies $\|\frac{1}{s}\sum_{j=1}^s h(\theta_j,x) - h_\mu(x)\|_{L^2(m)} \leq \epsilon$ for any probability measure $m$ on $\mathcal{X}$.
This means the set of majority vote classifiers $\mathcal{G}$ is a subset of the closure of ${\rm conv}(\mathcal{H})$ with respect to $L^2(m)$.

We now consider a general metric space $(\Omega,d)$. Let $A$ be an arbitrary subset of $\Omega$.
Let $\{B_{\epsilon}(z_i)\}_{i=1}^n$ be an $\epsilon$-open ball covering of $A$.
Then, $\{B_{r\epsilon}(z_i)\}_{i=1}^n$ ($\forall r>1$) is a covering of $\overline{A}$.
Let $w\in \overline{A}$ be an arbitrary point.
Since the covering $\{B_{\epsilon}(z_i)\}_{i=1}^n$ of $A$ is finite, we can obtain a sequence $\{w_k\}_{k=1}^\infty$ in $A$ such that $w_k\rightarrow w$ and $\{w_k\}_{k=1}^\infty$ is contained in a ball $\exists B_{\epsilon}(z_i)$.
This implies $d(w,z_i)\leq \epsilon$, specifically, $w \in B_{r\epsilon}(z_i)$.
Thus, we conclude the proof.
\end{proof}

Under Assumption \ref{covering_numbers_assumption}, the bound on the entropy of ${\rm conv}(H)$ is well known \cite{vw1996}, that is, there exists a positive constant $K$ such that
$\log N_m(\epsilon,{\rm conv}(H)) \leq K\epsilon^{-2V/(V+2)}$ for $\epsilon \in (0,1)$.
Combining the above proposition, we can conclude that the entropy $\log N_m(\epsilon,\mathcal{G})$ is also $O(\epsilon^{-2V/(V+2)})$.
Therefore, we can apply the result in \cite{kp2002}, and we immediately obtain the improved generalization bound.

\begin{theorem} [\cite{kp2002}]
Let us assume $\sup_m \log N_m(\epsilon,\mathcal{J}) = O(\epsilon^{-\alpha})$ for $\alpha \in (0,2)$, where supremum is taken over the set of all discrete measures on $\mathcal{X}$.
Then, there is a constant $K>0$ such that for $\forall \delta \in (0,1)$ with probability at least $1-\delta$ over the random choice of the $S$ for $\forall f \in \mathcal{J}$ we have
\[ \mathbb{P}_D[Yf(X)\leq 0] \leq K \inf_{\rho \in (0,1]} \left( \mathbb{P}_S[Yf(X)\leq \rho] + \rho^{\frac{-2\alpha}{2+\alpha}}N^{\frac{-2}{\alpha+2}} \right) + \frac{K}{N}\log\left(\frac{1}{\delta} \right). \]
\end{theorem}
\end{proof}

Next, we give proofs of the relation between smooth margin function and the empirical margin distribution.



\begin{proof} [ Proof of Theorem \ref{margin_distribution_bound_prop}. ]
Let $k$ be the number of examples whose margin is less than $\rho$, i.e., $yh_\mu(x)\leq \rho,\ (x,y)\in S$.
Then we have the following by considering potentially minimum of $\mathbb{E}_S[\exp(-yh_\mu(x)/\alpha)]$,
\[ \frac{k}{N}\exp\left(- \frac{\rho}{\alpha}\right) + \frac{N-k}{N}\exp\left(-\frac{1}{\alpha}\right) \leq \exp\left(- \frac{\psi_\alpha(\mu)}{\alpha}\right).\] 
Noting that $\frac{k}{N}=\mathbb{E}_S[{\bf 1}[yh_{\mu}(x)\leq \rho]]$, we can finish the proof of the theorem.
\end{proof}

\section{Topological Properties and Optimality Conditions}
In this section, we prove statements about the optimization problem for majority vote classifiers.

\begin{proof} [Proof of Proposition \ref{smoothness_proposition}. ]
By the assumption, uniform boundedness, Lipschitz continuity of $h(\cdot,x)$ and uniform boundedness of $\|\mathbb{E}_S[s_\mu(\theta,x,y)]\|_2^2$ are clear.
Thus, it is sufficient to show uniform Lipschitz continuity of the latter functions.
Let us define functions $\phi_{\alpha}(z)=\|\sum_{i=1}^N \alpha_iz_i/N \|_2^2$\ \ (where $z_i\in \mathbb{R}^d,\ \exists K,\ \alpha \in [-K,K]^N$) and mappings $\psi_S(\theta) = (\nabla_\theta h(\theta,x_i))_{i=1}^N$.
By the boundedness assumption there is a constant $C>0$ such that $\|\nabla_\theta h(\theta,x)\|_2 \leq C$.
Note that $\phi_\alpha|_{[-C,C]^{dN}}$ and $\psi_S$ are Lipschitz continuous with the uniformly bounded constant.
Thus, composite functions of these; $\{\phi_{\alpha}\circ \psi_S\}_{\alpha \in [-C,C]^{dN},S}$ are also Lipschitz continuous with the uniformly bounded constant.
Clearly, functions $\|\mathbb{E}_S[s_\mu(\theta,x,y)]\|_2^2$ is an element of these composite functions, so this finishes the proof.
\end{proof}

%

We now give propositions needed in our analysis.
The first statement in the following proposition shows that the distance between $\phi$ and $\phi+\xi \circ \phi$ with respect to $L^2(\mu)$ is the norm of $\xi$ with respect to $L^2(\phi_\sharp \mu)$.
The second statement gives a sufficient condition for a vector to define a diffeomorphism that preserves good properties if the base probability measures possesses these, for instance,
the absolute continuity with respect to Lebesgue measure and the manifold structure of the support of itself which are sometimes useful from the Wasserstein geometry or partial differential equation perspective.

\begin{proposition}\label{neighborhood_prop}
For $\mu \in \mathcal{P}_2$, the following statements are valid:\\
(i) $\| (id+\xi)\circ \phi - \phi\|_{L^2(\mu)} = \|\xi\|_{L^2(\phi_\sharp \mu)}$\ \ for $\phi \in L^2(\mu), \xi \in L^2(\phi_\sharp \mu)$;\\
(ii) Let $\xi \in L^2(\mu)$ be the $\mathcal{C}^1$-mapping from the convex hull of $\mathrm{supp}(\mu)$ to $\Theta$.
We denote by $\Lambda$ an upper bound on maximum singular values of $\nabla \xi(\theta)$ as the $(d,d)$-matrix on the convex hull of $\mathrm{supp}(\mu)$. 
Then $id+\eta\xi$ is a diffeomorphism on $\mathrm{supp}(\mu)$ for $0 \leq \forall \eta < 1/\Lambda$.
\end{proposition}

\begin{proof}
We set $\mu = \phi_\sharp \mu_0$ for $\phi \in L^2(\mu_0)$.
Then we have that for $\xi \in L^2(\mu)$

\begin{align*}
\|\xi\|_{L^2(\mu)}^2 &= \int \|\xi(\theta)\|_2^2 d\mu(\theta) \\
&= \int \|\xi(\theta)\|_2^2 d\phi_\sharp \mu_0(\theta) \\
&= \int \|\xi(\phi(\theta))\|_2^2d\mu_0(\theta) = \|\xi\circ \phi\|_{L^2(\mu_0)}^2,
\end{align*}

where we used the variable transformation for the third equality.
This finishes the proof of $(i)$.

If we assume $(id+\eta\xi)(\theta)=(id+\eta\xi)(\theta')$, then it follows that $\|\theta-\theta'\|_2=\eta\|\xi(\theta)-\xi(\theta')\|_2 \leq \eta\|\nabla \xi(\theta'')^T(\theta-\theta')\|_2< \eta \Lambda \|\theta-\theta'\|_2$, where $\theta''$ is a convex combination of $\theta$ and $\theta'$. 
Since $\eta\Lambda<1$, we have $\theta=\theta'$, i.e., $id+\eta\xi$ is an injective mapping.
By the same argument, we find that $\nabla (id+\eta\xi)(\theta) = I_d+\eta \nabla \xi(\theta)$ also defines an injective linear mapping for $\theta \in \mathrm{supp}(\mu)$ and $\eta\Lambda<1$, so that this matrix is invertible.
Thus, we conclude the proof of $(ii)$ by using the invertible mapping theorem.
\end{proof}

Here, we present the proof of Proposition \ref{continuity_prop_on_P} and the continuity of the parameterization via transport maps in the following propositions, which will be used to show a local optimality condition theorem.

\begin{proof} [ Proof of Proposition \ref{continuity_prop_on_P}. ]
Continuity of $h_\mu(x)$ and $\mathcal{L}_S(\mu)$ with respect to $\mu$ are clear.
Let $\{\mu_t\}_{t=1}^\infty$ be a sequence converging to $\mu \in \mathcal{P}$.
In the following, we denote $s_\mu(\theta,x,y)$ by $s_\mu$ for simplicity.
By triangle inequality, we have
\begin{align*}
  \left| \|\mathbb{E}_S[s_{\mu_t}]\|_{L^2(\mu_t)}^2 - \| \mathbb{E}_S[s_{\mu} ]\|_{L^2(\mu)}^2 \right|
  &\leq \left| \mu_t( \|\mathbb{E}_S[s_{\mu_t}]\|_2^2 ) - \mu_t( \|\mathbb{E}_S[s_{\mu}]\|_2^2 ) \right| \\
  &+ \left| \mu_t( \|\mathbb{E}_S[s_{\mu}]\|_2^2 ) - \mu( \|\mathbb{E}_S[s_{\mu}]\|_2^2 ) \right|.
\end{align*}
Since $\|\mathbb{E}_S[s_{\mu}]\|_2^2 \in \mathcal{F}$, the latter term converges to zero.
In order to show that the former converges to zero, it is sufficient to see the uniform convergence $\|\mathbb{E}_S[s_{\mu_t}]\|_2^2 \rightarrow \|\mathbb{E}_S[s_{\mu}]\|_2^2$.
By the boundedness and the triangle inequality, we have
\begin{align*}
\left| \|\mathbb{E}_S[s_{\mu_t}]\|_2^2 - \|\mathbb{E}_S[s_{\mu}]\|_2^2 \right| &\leq 2\sqrt{C} \left| \|\mathbb{E}_S[s_{\mu_t}]\|_2 - \|\mathbb{E}_S[s_{\mu}]\|_2\right| \\
&\leq 2\sqrt{C} \|\mathbb{E}_S[s_{\mu_t}] - \mathbb{E}_S[s_{\mu}]\|_2.
\end{align*}
This upper bound converges to zero.
Indeed, each element in expectation: $s_{\mu_t}(\cdot,x,y)$ uniformly converges to $s_{\mu}(\cdot,x,y)$ as seen in the following:
\[ \|l'(-y_{h_{\mu_t}}(x))\nabla h(\theta,x) - l'(-y_{h_{\mu}}(x))\nabla h(\theta,x)\|_2 \leq C|l'(-y_{h_{\mu_t}}(x))-l'(-y_{h_{\mu}}(x))| \rightarrow 0. \]
This finishes the proof.
\end{proof}

\begin{proposition}\label{continuity_prop_of_iota}
For $\forall \mu \in \mathcal{P}_2$ and $\forall \xi \in L^2(\mu)$, it follows that $d_{\mathcal{F}}((id+\xi)_\sharp\mu,\mu) \leq C \|\xi\|_{L^2(\mu)}$.
\end{proposition}

\begin{proof} [ Proof of Proposition \ref{continuity_prop_of_iota}. ]
Noting that Lipschitz continuity of $\forall f \in \mathcal{F}$, we have that for $\forall \xi \in L^2(\mu)$,
\begin{align*}
  d_{\mathcal{F}}((id+\xi)_\sharp \mu, \mu) &= \sup_{f\in\mathcal{F}}|((id+\xi)_\sharp \mu)(f)-\mu(f)| \\
  &= \sup_{f\in\mathcal{F}} \left| \int f(\theta)d(id+\xi)_\sharp\mu(\theta) - \int f(\theta)d\mu(\theta) \right| \\
  &= \sup_{f\in\mathcal{F}} \left| \int \left(f(\theta+\xi(\theta)) - f(\theta) \right)d\mu(\theta) \right| \\
  &\leq C  \int \|\xi(\theta)\|_2 d\mu(\theta) \leq C\| \xi \|_{L^2(\mu)},
\end{align*}
where we used H{\"o}lder's inequality for the last inequality.
\end{proof}

As noted in the paper, the continuity in Proposition \ref{continuity_prop_on_P} also holds with respect to $p$-Wasserstein distance ($p\geq 1$)
and Proposition \ref{continuity_prop_of_iota} holds for $1$-Wasserstein distance with $C=1$.

We now give the proof of the counterpart of Taylor's formula.

\begin{proof} [ Proof of Proposition \ref{asymptotic_equality_prop}. ]
By the variable transformation, we have
\begin{align*}
\int h(\theta,x) d(id+\xi)_{\sharp} \mu(\theta) =\int h(\theta+\xi(\theta),x)d\mu(\theta).
\end{align*}
Using Taylor's formula, we obtain
\[ h_{\psi\sharp \mu_0}(x) =h_\mu(\theta) +\mathbb{E}_\mu[\nabla_\theta h(\theta,x)^T\xi(\theta) + \|\xi(\theta)\|_{\nabla^2_\theta h(\theta',x)}^2], \]
where $\|\cdot\|_{\nabla_\theta h(\theta',x)}$ is Mahalanobis norm, and 
\[ l(a+b)=l(a)+l'(a)b + \frac{1}{2}l''(a)b^2 + o(b^2)\ \ \ (a,b\in \mathbb{R}). \]
Noting that by H\"older's inequality and Assumption \ref{smooth_bound_assumption}, $\mathbb{E}_\mu[\nabla_\theta h(\theta,x)^T\xi(\theta)]=O(\|\xi\|_{L^2(\mu)})$ and $\mathbb{E}_\mu[\|\xi(\theta)\|_{\nabla_\theta h(\theta',x)}^2]=o(\|\xi\|_{L^2(\mu)})$, we get
\[ l(-yh_{\psi_\sharp \mu_0}(x)) = l(-yh_\mu(x))+ \mathbb{E}_\mu[s_\mu(\theta,x,y)^T \xi(\theta)] + H_\mu(\xi,x,y) + o(\|\xi\|_{L^2(\mu)}^2),  \]

where $H_\mu(\xi,x,y)$ is the integrand in $H_\mu(\xi)$.
Therefore, by taking the expectation $\mathbb{E}_S$, we finish the proof.  
\end{proof}

Using facts and propositions presented in the paper, we prove the theorem of a necessary optimality condition.

\begin{proof} [ Proof of Theorem \ref{optimality_condition_thm}. ]
We denote $\zeta_\mu =\mathbb{E}_S[s_{\mu}(\cdot,x,y)]$ and denote the $\delta$-ball centered at $\mu_*$ by $B^\mathcal{F}_{\delta}(\mu_*)$ with respect to $d_{\mathcal{F}}$.
We assume $\mu_*$ is a minimum on $B^\mathcal{F}_{\delta}(\mu_*)$.
By Assumption \ref{smooth_bound_assumption} and Proposition \ref{continuity_prop_of_iota}, there exists $\eta_0>0$ such that $(id\pm\eta\zeta_\mu)_{\sharp}\mu \in B^\mathcal{F}_{\delta}(\mu_*)\cap \mathcal{P}_2$
for $0<\forall \eta <\eta_0$ and $\forall \mu \in B^\mathcal{F}_{\delta/2}(\mu_*)\cap \mathcal{P}_2$.
Let $\epsilon>0$ be an arbitrary constant.
Here, we can choose a sequence $\{\mu_t\}_{t=1}^\infty$ in $B^\mathcal{F}_{\delta/2}(\mu_*)\cap \mathcal{P}_2$ satisfying $\mu_t\rightarrow \mu_*$ and
$\mathcal{L}_S(\mu_t)\leq \mathcal{L}_S(\mu_*)+\epsilon/t$ by the continuity of $\mathcal{L}_S$.
Then, using Proposition \ref{asymptotic_equality_prop}, we have
\begin{align*}
-\frac{\epsilon}{t} &\leq \mathcal{L}_S\left(\left(id-\frac{\eta_0}{t}\zeta_{\mu_t}\right)_\sharp \mu_t \right) - \mathcal{L}_S(\mu_t) \\
&=  -\frac{\eta_0}{t}\|\mathbb{E}_S[s_{\mu_t}]\|_{L^2(\mu_t)}^2 + \frac{\eta_0^2}{t^2}O(\|\mathbb{E}_S[s_{\mu_t}]\|_{L^2(\mu_t)}^2),
\end{align*}
where we denote $s_{\mu_t}=s_{\mu_t}(\theta,x,y)$ for simplicity.
Note that Assumption \ref{smooth_bound_assumption} is essentially stronger than Assumption \ref{continuity_assumption} and the continuity of $\mathcal{L}_S(\mu)$
and $\|\mathbb{E}_S[s_{\mu}]\|_{L^2(\mu)}^2$ with respect to $\mu$ are valid by Proposition \ref{continuity_prop_on_P}.
Thus, multiplying $t$, taking the limit as $t\rightarrow \infty$, and using continuity, we have $\eta_0 \|\mathbb{E}_S[s_{\mu_*}]\|_{L^2(\mu_*)}^2 \leq \epsilon$.
Since $\epsilon$ is taken arbitrary and $\epsilon, \eta_0$ are independent of each other, we get $\|\mathbb{E}_S[s_{\mu_*}]\|_{L^2(\mu_*)}^2=0$
\end{proof}

We give the proof of Proposition \ref{positivity_prop} giving the justification for the assumption of Theorem \ref{sufficient_optimality_condition_thm}.

\begin{proof} [ Proof of Proposition \ref{positivity_prop}. ]
  For $\xi \in L^\infty(\mu)$ satisfying $\|\xi\|_{L^\infty(\mu)} < \epsilon$, convex combinations of $\theta$ and $\theta + \xi(\theta)$ for $\theta \in {\rm supp}(\mu)$ is contained in
  ${\rm supp}^\epsilon(\mu)$. 
Thus, we have
\begin{align*}
  \mathbb{E}_\mu[ \|\xi\|_{M_\mu(\theta')}^2] &= \int_{{\rm supp}(\mu)}\|\xi\|_{M_\mu(\theta')}^2d\mu(\theta)\\
  &\geq \int_{{\rm supp}(\mu)} \frac{\alpha}{2} \|\xi\|_2^2 d\mu(\theta) = \frac{\alpha}{2}\|\xi\|_{L^2(\mu)}^2.
\end{align*}
This finishes the proof.
\end{proof}

We now prove the theorem of sufficient optimality condition.

\begin{proof} [ Proof of Theorem \ref{sufficient_optimality_condition_thm}. ]
We suppose condition (\ref{opt_necessary_condition}) holds.
It follows that for $\forall \xi \in L^\infty(\mu_*)$, $|\mathbb{E}_{\mu_*}[\mathbb{E}_S[s_{\mu_*}(\cdot,x,y)]^T\xi(\theta) ]| \leq \|\mathbb{E}_S[s_{\mu_*}(\cdot,x,y)]\|_{L^2(\mu_*)}\|\xi\|_{L^2(\mu_*)}=0$ by H\"older's inequality.
By the assumption, we can choose $\epsilon>0$ such that $\frac{1}{4}H_{\mu_*}(\xi)\geq |o(\|\xi\|_{L^2(\mu_*)}^2)|$ on $\{\|\xi\|_{L^\infty(\mu_*)}<\epsilon \}$, where the right hand side is the higher-order term in (\ref{asymptotic_equality}).
These inequalities imply that if $\|\xi\|_{L^\infty(\mu_*)}<\epsilon$,
\[ \mathcal{L}_S((id+\xi)_\sharp \mu_*) \geq \mathcal{L}_S(\mu_*) + \frac{1}{4}H_{\mu_*}(\xi) > \mathcal{L}_S(\mu_*). \]
This finishes the proof.
\end{proof}

\section{Interior Optimality Property}
To prove Theorem \ref{characterization_optimality_condition_thm}, we now introduce the notion of the smoothing of probability measures as Schwartz distribution
We denote by $\chi$ a $\mathcal{C}^\infty$-class probability density function on $\Theta=\mathbb{R}^d$ with $\mathrm{supp}(\chi)=\{\theta \in \Theta \mid \|\theta\|_2\leq 1\}$ 
and write $\chi_\epsilon(\theta)=\epsilon^{-d}\chi(\theta/\epsilon)$ for $\epsilon>0$.
For a probability measure $\mu \in \mathcal{P}$, it can be approximated by a smooth probability density function defined by the following:

\begin{equation*}
  (\mu * \chi_\epsilon)(\theta) = \int_{\Theta} \chi_\epsilon(\theta-\theta') d\mu(\theta').
\end{equation*}

It is well known that $\mu * \chi_\epsilon$ is $\mathcal{C}^\infty$-class on $\Theta$ and converges as Schwartz distribution to $\mu$ as $\epsilon \rightarrow 0$ \cite{hormander1963}.
Moreover, if $\mu$ possesses a $L^p$-integrable probability density function $q \in L^p(\Theta)$ with $p\geq 1$, then $\mu * \chi_\epsilon$ converges to $q$ with respect to $L^p(\Theta)$-norm.
Let $\mu_\epsilon$ denote a probability measure induced by $\mu * \chi_\epsilon$.
When $\mathrm{supp}(\mu_\epsilon)$ is compact in $\Theta$, $\mu_\epsilon(f)$ converges to $\mu(f)$ for arbitrary continuous function $f$ on $\Theta$.
This can be confirmed by constructing a $\mathcal{C}^\infty$-function $g$ that uniformly approximates $f$ on $\mathrm{supp}(\mu_\epsilon)$ and takes the value zero outside of sufficiently large compact set.
Clearly, we see that $\mathrm{supp}(\mu_\epsilon)$ is contained in the closed $\epsilon$-neighborhood of $\mathrm{supp}(\mu)$.
Thus, if $\mathrm{supp}(\mu)$ is compact, then $\{\mu_\epsilon\}_{\epsilon \in (0,1)}$ is tight, so that we can find $\mu_\epsilon$ converges to $\mu$ with respect to $\|\cdot\|_{\mathcal{F}}$
by the proof of Proposition \ref{weak_convergence_prop}, that is, $\mu_\epsilon(f)$ converges uniformly to $\mu(f)$ on $\mathcal{F}$.

Note that if $\mathrm{supp}(\mu)$ is the compact submanifold in $\Theta$, $\mathrm{supp}(\mu)$ and the closed $\epsilon$-neighborhood of $\mathrm{supp}(\mu)$ coincide
for sufficiently small $\epsilon>0$
and these sets possess a manifold structure as can be seen by the following auxiliary lemma. 

\begin{lemma}
  Let $\mathcal{M}$ be a $l$-dimensional compact $\mathcal{C}^\infty$-submanifold ($l<d$) or a $d$-dimensional compact $\mathcal{C}^\infty$-submanifold with boundary in $\mathbb{R}^d$.
  If $\epsilon>0$ is sufficiently small, then closed $\epsilon$-neighborhood of $\mathcal{M}$ in $\mathbb{R}^d$ is a $d$-dimensional compact $\mathcal{C}^\infty$-submanifold with boundary.
\end{lemma}
\begin{proof}
  We only prove the case where $\mathcal{M}$ is a compact $\mathcal{C}^\infty$-submanifold since we can give a proof for a $d$-dimensional compact $\mathcal{C}^\infty$-submanifold with boundary in a similar fashion.
  Let $\mathcal{M}^\epsilon$ denote an open $\epsilon$-neighborhood of $\mathcal{M}$ in $\Theta=\mathbb{R}^d$.
  By the $\epsilon$-neighborhood theorem \cite{gp1974}, if $\epsilon$ is sufficiently small, then $\forall \theta \in \mathcal{M}^\epsilon$ possesses a unique closest point $\pi(\theta)$ in $\mathcal{M}$ and the map $\pi: \mathcal{M}^\epsilon \rightarrow \mathcal{M}$ is a submersion.
  Moreover, for each $\theta_0 \in \mathcal{Y}$, we can see that there is a local coordinate system $z=(z^1,\ldots,z^d)=\phi(\theta)$ on an open subset $U\subset \Theta$
  such that $\phi(\theta_0)=(0,\ldots,0)$, $\phi(\mathcal{M} \cap U )= \{z \in \phi(U) \mid z^{l+1}=\cdots=z^{d}=0 \}$, and the submersion $\pi$ can be written as
  $\pi(\phi^{-1}(z))=\phi^{-1}(z^1,\ldots,z^l,0,\ldots,0)$ for $z \in \phi( \mathcal{M}^\epsilon \cap U)$.
  Since, $\mathcal{M}$ is compact, the closed $\epsilon$-neighborhood $\overline{\mathcal{M}^\epsilon}$ is covered by a finite number of such local coordinate systems for sufficiently small $\epsilon>0$.
  We redefine $(U,\phi)$ to be one of such local coordinate system.
  The Euclidean distance to $\mathcal{M}$ from $\phi^{-1}(z)\in U$ is $f(z)=d(\phi^{-1}(z),\mathcal{M})=\| \phi^{-1}(z) - \phi^{-1}(z^1,\ldots,z^l,0,\ldots,0) \|_2$ and
  $\overline{\mathcal{M}^\epsilon}\cap U$ is represented as $\{ \phi^{-1}(z) \mid z\in \phi(U), f(z)\leq \epsilon \}$.
  Since, $f(\cdot)$ is a $\mathcal{C}^\infty$-function and $df \neq 0$ on a neighborhood of $\partial \overline{\mathcal{M}^\epsilon}$ in $U$, $\overline{\mathcal{M}^\epsilon}\cap U$
  is a $d$-dimensional compact $\mathcal{C}^\infty$-submanifold with boundary in $\Theta$.
\end{proof}

Let $U$ be a bounded domain with smooth boundary in $\Theta=\mathbb{R}^d$, that is, $\overline{U}$ is a $d$-dimensional $\mathcal{C}^\infty$-manifold with boundary.
We denote by $H^1(U)(=W^{1,2}(U))$ the Sobolev space and we denote by $V(U)$ a linear subspace $\{ f \in H^1(U) \mid \int_U f d\theta=0\}$.
We equip $H^1(U)$ with the Sobolev inner product $\pd<u,v>_{H^1(U)} = \int_U u(\theta)v(\theta) d\theta + \int_U \nabla u(\theta)^\top \nabla v(\theta) d\theta$ and 
we equip $V(U)$ with the inner product $\pd< u, v>_{V(U)}=\int_U \nabla u(\theta)^\top \nabla v(\theta)d\theta$, ($u,v \in V(U)$).
The non-degeneracy and the completeness of $\pd<,>_{V(U)}$ on $V(U)$ can be checked as follows. 
We denote $\overline{u}=\int_U u(\theta) d\theta / |U|$, where $|U|$ is the Lebesgue measure of $U$.
Since $\overline{u}=0$ for $u \in V(U)$, we get from the Poincar{\' e}-Wirtinger inequality that there exists $C_U>0$ such that
\begin{equation}
  \|u\|_{L^2(U)}=\| u-\overline{u}\|_{L^2(U)} \leq C_U \|\nabla u\|_{L^2(U)} = C_U\|u\|_{V(U)}. \label{norm_inequality}
\end{equation}
Thus, we have
\begin{equation*}
  \|u\|_{V(U)} \leq \|u\|_{H^1(U)} = \sqrt{ \|u\|_{L^2(U)}^2 + \|\nabla u\|_{L^2(U)}^2} \leq (1+C_U) \|u\|_{V(U)}. 
\end{equation*}
This inequality means that these two norms introduce the same topology to $V(U)$ and it immediately implies the non-degeneracy and also the completeness of $\|\cdot\|_{V(U)}$ on $V(U)$ because
$V(U)$ is the closed subspace in the Sobolev space $H^1(U)$ with respect to $\|\cdot\|_{H^1(U)}$.
Therefore, $V(U)$ with $\pd<,>_{V(U)}$ is actually Hilbert space.

Although, the Poincar{\' e} constant $C_U$ depends on a region $U$, it is known that for any $R>0$ there exists $C_R>0$ such that if $U$ is an $\epsilon$-open neighborhood of a connected
set $K \subset B_R(0)=\{\theta \in \Theta \mid \|\theta\|_2 < R\}$ for some constant $\epsilon>0$, then $C_U$ can be taken as it is upper bounded by $C_R$ \cite{ruiz2012}.

In our analysis, we need an estimation of the norm of a solution to the problem where for $f \in V(U)$, the task is to find $u \in V(U)$ satisfying the following equation:
\begin{equation}
\int_U \nabla u(\theta)^\top \nabla v(\theta) d\theta = -\int_U f(\theta) v(\theta) d\theta \ \ for\ any\ v \in V(U).  \label{weak_neumann_prob}
\end{equation}
This is the weak formulation of the {\it Neumann problem}: to find $u\in V(U)$ such that $\Delta u = f$ in $U$ and $\partial u/ \partial n = 0$ on $\partial U$,
where $n$ is the outward pointing unit normal vector of $\partial U$.
An upper bound on the norm of a solution is given by the following lemma which can be proven in the standard way in partial differential equation theory.

\begin{lemma} \label{norm_of_solution_bound}
  Let $U$ be a bounded domain in $\Theta=\mathbb{R}^d$.
  Then for any $f \in V(U)$, a solution $u_* \in V(U)$ to the problem (\ref{weak_neumann_prob}) exists and its norm is bounded as follows:
  \begin{equation}
    \| u_* \|_{V(U)} \leq \| \alpha_f \|_{V(U)^*},
  \end{equation}
  where $\alpha_f$ is a linear functional $\alpha_f(u)=\int_U f(\theta)u(\theta)d\theta$ $(u\in V(U))$ and $\|\cdot\|_{V(U)^*}$ denote the dual of the norm $\|\cdot\|_{V(U)}$.
\end{lemma}
\begin{proof}
  We denote $\beta(u,v) = \int_U \nabla u(\theta)^\top \nabla v(\theta) d\theta$ for $u,v \in V(U)$.
  Clearly, $\beta(\cdot,\cdot)$ is bilinear function.
  The boundedness with respect to $\|\cdot\|_{V(U)}$ are shown as follows.
  Using H{\"o}lder's inequality and the inequality (\ref{norm_inequality}), we have that for $u,v \in V(U)$,
  \[ | \beta(u,v) | = \left| \int_U \nabla u(\theta)^\top \nabla v(\theta) d\theta \right| \leq \|u\|_{L^2(U)}\|v\|_{L^2(U)} \leq C_U^2\|u\|_{V(U)}\|v\|_{V(U)}.  \]
  Moreover, $\beta(\cdot,\cdot)$ is $1$-coercive because $\beta(u,u)=\|u\|^2_{V(U)}$.
  We can also see that $\alpha_f(\cdot)$ is a bounded linear functional in the same manner: for $u \in V(U)$,
  \[ | \alpha_f(u) | = \left| \int_U f(\theta )u(\theta) d\theta \right| \leq \|f\|_{L^2(U)}\|u\|_{L^2(U)} \leq C_U\|f\|_{L^2(U)}\|u\|_{V(U)}.  \]
  Thus, by the Lax-Milgram theorem, there is a unique solution $u_* \in V(U)$ and we have $\|u_*\|_{V(U)}\leq \|\alpha_f\|_{V(U)^*}$.
\end{proof}

We now give the proof of Theorem \ref{characterization_optimality_condition_thm} that gives an interior optimality property of the local optimality condition.
  
\begin{proof} [ Proof of Theorem \ref{characterization_optimality_condition_thm}. ]
  We denote $\Omega = {\rm supp}(\mu_*)$ and let $q_*$ be a continuous probability density function of $\mu_*$. 
  We assume that there exists $\mu' \in \mathcal{P}$ that possesses a continuous probability density function $q'$ and satisfies ${\rm supp}(\mu') \subset \Omega$, $\mathcal{L}_S(\mu') < \mathcal{L}_S(\mu_*)$.  
  By smoothing $\mu_*$ and $\mu'$ with sufficiently small $\epsilon>0$, we can obtain $d\mu'_{\epsilon} = q'_\epsilon(\theta)d\theta$ and $d\mu_{*\epsilon} = q_{*\epsilon}(\theta)d\theta$
  where $q'_\epsilon, q_{*\epsilon}$ are $\mathcal{C}^\infty$-density functions satisfying ${\rm supp}(\mu'_\epsilon)\subset {\rm supp}(\mu_{*\epsilon})$.
  As stated above, $q'_\epsilon, q_{*\epsilon}$ converge to $q', q_*$ in $L^2(\Theta)$.
  
  Let us denote $\Omega_{\epsilon} = {\rm supp}(\mu_{*\epsilon})$
  Since $q'_\epsilon-q_{*\epsilon}$ is $\mathcal{C}^\infty$-function and $\int_{\Omega_\epsilon} (q'_\epsilon-q_{*\epsilon}) d\theta = 0$, there is a $\mathcal{C}^\infty$-function $\psi_\epsilon$ on $\Omega_\epsilon$
  that solves the Neumann problem \cite{hormander1963}:
  \[ \Delta \psi_\epsilon = q'_\epsilon-q_{*\epsilon} \ \ in\ \  \Omega_\epsilon,\ \ \partial \psi_\epsilon/\partial n = 0\ \  on\ \  \partial \Omega_\epsilon, \]
  where $\partial \Omega_\epsilon$ is the boundary of $\Omega_\epsilon$ and $n$ is the outward pointing unit normal vector of $\partial \Omega_\epsilon$.
  By adding a constant, we assume $\int_{\Omega_\epsilon}\psi_\epsilon(\theta)d\theta=0$, i.e., $\psi_\epsilon|_{\Omega^i_\epsilon} \in V(\Omega^i_\epsilon)$,
  where $\Omega^i_\epsilon$ is the interior of $\Omega_\epsilon$.
  Therefore, we have
  \begin{align}
    \int_{\Omega_\epsilon} \nabla_\mu \mathcal{L}_S(\mu_{*\epsilon})(\theta)d(\mu'_\epsilon-\mu_{*\epsilon}) &= \int_{\Omega_\epsilon} \nabla_\mu \mathcal{L}_S(\mu_{*\epsilon})(\theta) \Delta \psi_\epsilon(\theta) d\theta \notag\\
    &= - \int_{\Omega_\epsilon} \nabla_\theta\nabla_\mu \mathcal{L}_S(\mu_{*\epsilon})(\theta)^\top \nabla_\theta \psi_\epsilon(\theta) d\theta \notag\\
    &+\int_{\partial \Omega_\epsilon} \nabla_\mu \mathcal{L}_S(\mu_{*\epsilon})(\theta) \frac{\partial \psi_\epsilon(\theta)}{\partial n} d\partial \Omega_{\epsilon} \notag \\
    &= - \int_{\Omega_\epsilon} \mathbb{E}_S[s_{\mu_{*\epsilon}}(\theta,x,y)]^\top \nabla_\theta \psi_\epsilon(\theta) d\theta, \label{weak_equation}
  \end{align}
  where for the second equality we used Green's formula and for the last equality we used $\partial \psi_\epsilon/ \partial n = 0$.
  By the convexity of $\mathcal{L}_S$ with respect to $\mu$ in terms of Affine geometry and $\mathcal{L}_S(\mu') < \mathcal{L}_S(\mu_{*})$, we have
  \begin{equation}
    \lim_{\epsilon \rightarrow 0}\int_{\Omega_\epsilon} \nabla_\mu \mathcal{L}_S(\mu_{*\epsilon})(\theta)d(\mu'_\epsilon-\mu_{*\epsilon})
  \leq \lim_{\epsilon \rightarrow 0} \mathcal{L}_S(\mu'_\epsilon)- \mathcal{L}_S(\mu_{*\epsilon})
  = \mathcal{L}_S(\mu')- \mathcal{L}_S(\mu_{*}) < 0. \label{lim_of_optimality}
  \end{equation}
  By the boundedness of $\Omega$, we can assume it is contained in a ball with radius $R>0$ centered around 0.
  Since, $\psi_\epsilon$ solves (\ref{weak_neumann_prob}) with $U=\Omega^i_\epsilon$ and $f = q'_\epsilon - q_{*\epsilon}$, we get that by Lemma \ref{norm_of_solution_bound}, 
  \begin{align}
    \lim_{\epsilon \rightarrow 0} \| \psi_\epsilon \|_{V(\Omega^i_\epsilon)} &\leq \lim_{\epsilon \rightarrow 0} \sup_{\|u\|_{V(\Omega^i_\epsilon)}\leq 1} \left| \int_{\Omega^i_\epsilon} (q'_\epsilon - q_{*\epsilon})(\theta)u(\theta)d\theta \right| \notag \\
    & \leq \lim_{\epsilon \rightarrow 0} \sup_{\|u\|_{V(\Omega^i_\epsilon)}\leq 1} \|q'_\epsilon - q_{*\epsilon}\|_{L^2(\Omega^i_\epsilon)} \|u\|_{L^2(\Omega^i_\epsilon)} \notag \\
    & \leq \lim_{\epsilon \rightarrow 0} \|q'_\epsilon - q_{*\epsilon}\|_{L^2(\Omega^i_\epsilon)} \sup_{\|u\|_{V(\Omega^i_\epsilon)}\leq 1} C_{\Omega^i_\epsilon}\|u\|_{V(\Omega^i_\epsilon)} \notag \\
    & \leq C_R \|q'-q_*\|_{L^2(\Theta)}, \notag
  \end{align}
  where we used the Poincar{\' e}-Wirtinger inequality (\ref{norm_inequality}) and uniform boundedness of $C_{\Omega^i_\epsilon}$.
  Thus, the limit as $\epsilon \rightarrow 0$ in the right hand side of (\ref{weak_equation}) is lower-bounded by
  \begin{equation}
    - \lim_{\epsilon \rightarrow 0} \| \mathbb{E}_S[s_{\mu_{*\epsilon}}(\theta,x,y)] \|_{L^2(\Omega_\epsilon)} C_R\|q'-q_*\|_{L^2(\Theta)} = - \| \mathbb{E}_S[s_{\mu_{*}}(\theta,x,y)] \|_{L^2(\Omega)} C_R\|q'-q_*\|_{L^2(\Theta)}. \label{lim_of_optimality2}
  \end{equation}
  Combining (\ref{lim_of_optimality}) and (\ref{lim_of_optimality2}), we find $\mathbb{E}_S[s_{\mu_*}(\theta,x,y)] \centernot\equiv 0$ on $\Omega$,
  so $\mu_*$ does not satisfy the local optimality condition (\ref{opt_necessary_condition}).
  This finishes the proof of the theorem.

  For the case where $\mu$ does not have a continuous density, we can show the same result in a similar way by smoothing $\mu$ with as its support is contained in $\Omega$.
\end{proof}

\section{Convergence Analysis}
In this section, we prove the convergence theorem of the proposed method.

\begin{proof} [ Proof of Lemma \ref{descent_lemma}. ]
Putting $\eta \zeta$ into (\ref{asymptotic_equality}) and (\ref{higher_order_bound}), we obtain
\begin{align*}
\mathcal{L}_S((id+\eta\zeta)_\sharp \mu) 
&\leq \mathcal{L}_S(\mu)- \eta \mathbb{E}_\mu[\mathbb{E}_S[s_\mu(\theta,x,y)]^TA_\mu(\theta)^{-1}s_\mu(\theta,x',y')] \\
&+\frac{\eta^2}{2}\mathbb{E}_\mu[\|s_\mu(\theta,x',y')\|_{A_\mu(\theta)^{-1}}^2].
\end{align*}
Note that by the assumption, there exists $G>0$ such that $\mathbb{E}_\mu[\|s_\mu(\theta,x',y')\|_{A_\mu(\theta)^{-1}}^2] < G$.
Moreover, using the bound on $A_\mu(\theta)^{-1}$ and taking the expectation with respect to $(x',y')$ (i.e., $\mathbb{E}_S$), we can finish the proof.  
\end{proof}

\begin{proof} [ Proof of Theorem \ref{convergence_thm}. ]
  Using the Lemma \ref{descent_lemma}, we can see the updates of Algorithm \ref{SPGD} decreases the objective value as follows:
  \[ \mathbb{E}_S[\mathcal{L}_S(\mu_{k+1})] \leq \mathcal{L}_S(\mu_{k})-\eta \|\mathbb{E}_S[s_{\mu_k}(\theta,x,y)]\|_{L^2(\mu_k)}^2 + \eta^2 G. \]
  Taking an expectation of the history of sample, summing up $k\in \{1,\ldots,t-1\}$, and dividing by $t\eta$, we have
  \[ \frac{1}{t}\sum_{k=1}^t \mathbb{E}[\|\mathbb{E}_S[s_{\mu_k}(\theta,x,y)]\|_{L^2(\mu_k)}^2] \leq \frac{\mathcal{L}_S(\mu_0)-\inf_{\mu\in Q} \mathcal{L}_S(\mu)}{\eta t} + \eta G. \]

  Thus, if $t \geq \frac{2(\mathcal{L}_S(\mu_0)-\inf_q \mathcal{L}_S(\mu)) }{\epsilon\eta}$,
  then $\frac{1}{t}\sum_{k=1}^t \mathbb{E}[\|\mathbb{E}_S[s_{\mu_k}(\theta,x,y)]\|_{L^2(\mu_k)}^2] \leq \epsilon$. 
  This means the method can find $\epsilon$-accurate solution with respect to the expectation, up to $t$ iterations.
\end{proof}


\section{Other Perspectives of SPGD} \label{sec:other_perspectives}
In this section, we provide two perspectives of SPGD: one is the functional gradient method in $L^2(\mu_0)$ where $\mu_0$ is the fixed initial probability measure in the method
and the other is the discretization of the continuous curve satisfying the gradient flow in the space of probability measure.
To describe the former perspective, we need the notion of the continuity equation which characterizes a curve of probability measures.

\subsection{Discretization of Gradient Flow Perspective}

\subsection*{Continuity Equation and its Discretization}
In Euclidean space, the step of the steepest descent method for optimization problems can be derived by the discretization of a continuous curve satisfying the gradient flow defined by the objective function.
To make a similar argument in the space of probability measures in a rigorous way, we need the continuity equation that characterizes a curve of probability measures and the tangent space where velocities of curves should be contained (c.f., \cite{ags2008}).
Though, we can more directly derive and analyze our method (proposed later) without these notions which requires a bit complicated definitions, 
it will help understanding of the dynamics of the method, so we here briefly introduce it with simplified arguments.
We refer to \cite{ags2008} for detailed descriptions in this direction and also refer to \cite{otto2001,pd2009,bjm2015} which follow an original fashion developed by Otto.

For $\mu \in \mathcal{P}$, let $\{\phi_t\}_{t\in [0,\delta]}$ be a curve in $L^2(\mu)$ that solves the following ordinary differential equation: for a vector field $v_t$ on $\Theta$, 
\begin{equation*}
  \phi_0 = id,\ \ \frac{d}{dt}\phi_t(\theta) = v_t( \phi_t(\theta))\ for \ \forall \theta \in \Theta.
\end{equation*}
We set $\nu_t=\phi_{t\sharp}\mu$.
For simplicity, we assume that $\nu_t$ have smooth density functions $d\nu_t/d\theta$ with respect to $t\in [0,\delta]$.
We denote by $\mathcal{C}_{c}^\infty(\Theta)$ the set of $\mathcal{C}^\infty$-functions with compact support in $\Theta=\mathbb{R}^d$.
Using integration by parts, we have that for $\forall f \in \mathcal{C}_{c}^\infty(\Theta)$,
\begin{align}
  \int_{\Theta} f(\theta) \frac{d}{dt} \frac{d\nu_t}{d\theta}d\theta &= \frac{d}{dt}\int_{\Theta} f(\theta)d\nu_t(\theta) \notag \\
  &=\frac{d}{dt}\int_{\Theta} f(\phi_t(\theta))d\mu(\theta) \notag \\
  &=\int_{\Theta} \nabla f(\phi_t(\theta))^\top \frac{d\phi_t(\theta)}{dt} d\mu(\theta) \notag \\
  &=- \int_{\Theta} f(\theta) \nabla \cdot (v_t d\nu_t), \label{weak_derivative}
\end{align}
where $\nabla\cdot$ is the divergence operator in the weak sense.
That is, this equation means the equality between $d\nu_t/dt$ and $- \nabla \cdot (v_t d\nu_t)$ as the distribution on $\mathcal{C}_{c}^\infty(\Theta)$, 
and indicates that the vector field $v_t$ controls the local behavior of $\nu_t$.
In general, for a Borel family of probability measures $\nu_t$, on $\Theta$, defined for $t$ in the open interval $I\subset \mathbb{R}$
and for a Borel vector field $v: (\theta,t) \rightarrow v_t(\theta) \in \Theta$, the following distribution equation in $\Theta \times I$,
\begin{equation}
  \frac{d}{dt}\nu_t + \nabla \cdot (v_t\nu_t) = 0 \label{continuity_equation}
\end{equation}
is called the {\it continuity equation}, i.e., for $\forall f \in \mathcal{C}^{\infty}_c(\Theta \times I )$,
\begin{equation}
  \int_{I} \int_{\Theta} (\partial_t f(\theta,t) + \nabla_{\theta}f(\theta,t)^\top v_t) d\nu_t dt = 0.
\end{equation}

Let $\mathcal{P}_2$ be equipped with the $2$-Wasserstein distance $W_2$. 
We again refer to \cite{villani2009,ags2008} for the definitions related to Wasserstein geometry.
Noting that any divergence-free vector field $w \in L^2(\mu)$ (i.e., $\nabla\cdot(w\mu)=0$) has no effect on $d\nu_t/dt |_{t=0}$ in the continuity equation, it is natural to consider the equivalence class of
$v \in L^2(\mu)$ modulo divergence-free vector fields and there exists a unique $\Pi(v)$ that attains the minimum $L^2(\mu)$-norm in this class :
$\Pi(v)=\arg\min_{w \in L^2(\mu)} \{ \|w\|_{L^2(\mu)} \mid \nabla\cdot((v-w)\mu)=0\}$.
We here introduce the definitions of the tangent space at $\mu \in \mathcal{P}_2$ as follows:
\begin{equation*}
  T_\mu \mathcal{P}_2 \overset{\mathrm{def}}{=} \{ \Pi(v) \mid v \in L^2(\mu)\}.
\end{equation*}
Clearly, we can also see $T_\mu \mathcal{P}_2 \simeq L^2(\mu)/\sim$, where $\sim$ denotes the above equivalence relation.
Moreover, it is known that $\Pi$ is the orthogonal projection onto $T_\mu \mathcal{P}_2$ with respect to $L^2(\mu)$-inner product, that is, $\int_\Theta v^\top (w-\Pi(w))d\mu=0$ for every $v \in T_\mu\mathcal{P}_2$ and $w \in L^2(\mu)$.
We naturally equip $T_\mu \mathcal{P}_2$ with $L^2(\mu)$-inner product for $\mu \in \mathcal{P}_2$.
For $\phi \in L^2(\nu)$ and $\psi \in L^2(\mu)$ where $\mu=\phi_\sharp \nu$, we easily have $\|\psi\|_{L^2(\mu)} = \|\psi\circ\phi\|_{L^2(\nu)}$, so $\psi \circ \phi \in L^2(\nu)$.
Especially, we have $\pd< \xi, \zeta>_{L^2(\mu)} = \pd< \xi\circ \phi, \zeta \circ \phi>_{L^2(\nu)}$ for $\forall \xi, \forall \zeta \in L^2(\mu)$ by the change of variables.
Thus, the inner-product on the tangent space changes depending on the base point $\mu \in \mathcal{P}_2$ and it means $\mathcal{P}_2$ is heterogeneous and
has an infinite-dimensional Riemannian manifold-like structure as pointed out by \cite{otto2001}.

The continuity equation (\ref{continuity_equation}) characterizes the class of absolutely continuous curves in $\mathcal{P}_2$.
Indeed, for arbitrary continuous curve $\nu_t\ (t\in I)$ in $\mathcal{P}_2$ with respect to the topology of weak convergence, the absolutely continuity of the curve $\nu_t$
and satisfying the continuity equation (\ref{continuity_equation}) for some Borel vector field $v_t$ is equivalent, moreover, we can take $v_t$ from $T_\mu\mathcal{P}_2$ uniquely for almost everywhere $t\in I$ to satisfy the equation (\ref{continuity_equation}).
The following proposition shows how a perturbation using $v_t$ can discretizes an absolutely continuous curve $\nu_t$ and how $v_t$ approximates optimal transport maps locally.

\begin{proposition}[\cite{ags2008}] \label{ac_prop}
  Let $\nu_t: I \rightarrow \mathcal{P}_2$ be an absolutely continuous curve satisfying the continuity equation with a Borel vector field $v_t$ that is contained in $T_{\nu_t}\mathcal{P}_2$ almost everywhere $t \in I$.
  Then, for almost everywhere $t \in I$ the following property holds:
  \begin{equation*}
    \lim_{\delta \rightarrow 0} \frac{W_2(\nu_{t+\delta},(id+\delta v_t)_\sharp \nu_t)}{|\delta|} = 0.
  \end{equation*}
  In particular, for almost everywhere $t\in I$ such that $\nu_t$ is absolutely continuous with respect to Lebesgue measure we have
  \begin{equation*}
    \lim_{\delta \rightarrow 0} \frac{1}{\delta} (\mathbf{t}_{\nu_t}^{\nu_{t+\delta}}- id) = v_t \quad in\ L^2(\nu_t),
  \end{equation*}
  where $\mathbf{t}_{\nu_t}^{\nu_{t+\delta}}$ is the unique optimal transport map between $\nu_t$ and $\nu_{t+\delta}$.
\end{proposition}

This proposition suggests the update $\mu^+ \leftarrow (id+\xi)_\sharp \mu$ to discretize a curve in $\mathcal{P}_2$.
Though, the above approximation is justified only for tangent vectors $\xi \in T_\mu \mathcal{P}_2$ in the proposition, we do not need such an explicit restriction in our analyses,
so we choose $\xi$ from the whole space $L^2(\mu)$ rather than $T_\mu\mathcal{P}_2$, in this update.
Note that, when $\mu=\phi_\sharp \nu$, ($\nu \in \mathcal{P}_2, \phi \in L^2(\nu)$), the transport map $\phi$ is updated along $\xi \in L^2(\mu)$ as follows:
\begin{equation}
  \phi^+ \leftarrow \phi + \xi\circ \phi = (id+\xi)\circ \phi, \label{update_of_transport_map}
\end{equation}
and we can see $\phi^+ \in L^2(\nu)$ and it corresponds to $\mu^+$, i.e., $\phi^+_\sharp \nu = \mu^+$.
This means the discretization of a curve in $\mathcal{P}_2$ can be realized by the above update of transport maps.
The resulting problem is how to choose $\xi$ to solve the problem (\ref{abst_prob}) which is described precisely in Section \ref{sec:spgd}.

\subsection*{Discretization of Gradient Flow}
We here describe the gradient flow perspective in $\mathcal{P}_2$ which is the most straightforward way to understand our method.
We have explained that an absolutely continuous curve $\{\nu_t\}_{t\in I} $in $\mathcal{P}_2$ is well characterized by the continuity equation (\ref{continuity_equation}) and we have seen that
$\{v_t\}_{t \in I}$ in (\ref{continuity_equation}) corresponds to the notion of the velocity field induced by a curve in the space of transport maps.
Such a velocity points in the direction of the particle flow.
On the other hand, the Fr\'{e}chet differential $\mathbb{E}_S[s_\mu(\cdot,x,y)]$ points in an opposite direction to reduce the objective $\mathcal{L}_S$ at each particle.
Thus, these two vector fields exist in the same space and it is natural to consider the following equation:
\begin{equation}
  v_t = -\mathbb{E}_S[s_{\nu_t}(\cdot,x,y)]. \label{grad_flow}
\end{equation}
This equation for an absolutely continuous curve is called the gradient flow \cite{ags2008} and a curve satisfying this will reduce the objective $\mathcal{L}_s(\mu)$.
Indeed, we can find by chain rule \cite{ags2008} such a curve $\{\nu_t\}_{t \in I}$ also satisfies the following:
\begin{equation*}
  \frac{d}{dt} \mathcal{L}_S(\nu_t) = - \| \mathbb{E}_S[s_{\nu_t}(\cdot,x,y)] \|_{L^2(\nu_t)}^2.
\end{equation*}

Recalling that $\nu_t$ can be discretized well by $\nu_{t+\delta} \sim (id+\delta v_t)_\sharp \nu_t$, we notice that Algorithm \ref{SPGD} is a stochastic approximation for the discretization of $\nu_t$ 
satisfying the gradient flow (\ref{grad_flow}).

\subsection{Functional Gradient Descent Perspective}
Though, we have introduced our method to optimize a probability measure, it also can be readily recognized as the method to optimize a transport map in $L^2(\mu_0)$ if we fix the initial distribution $\mu_0 \in \mathcal{P}_2$.
Indeed, since a composite function $\phi \circ \psi$ is contained in $L^2(\mu_0)$ when $\psi \in L^2(\mu_0)$ and $\phi \in L^2(\psi_\sharp \mu_0)$, so obtained transport maps by Algorithm \ref{SPGD} also belong to $L^2(\mu_0)$.
Thus objective function can be translated to the form of $\mathcal{L}_S(\phi)=\mathcal{L}_S(\phi_\sharp \mu_0)$ with respect to $\phi \in L^2(\mu_0)$.
Note that in general, since an initial distribution is usually variable in several trials, such a translation does not make sense.

In a similar manner to the proof of Proposition \ref{asymptotic_equality_prop}, we can obtain the following formula: for $\phi, \tau \in L^2(\mu_0)$,
\begin{align*}
\mathcal{L}_S(\phi+\tau)&=\mathcal{L}_S(\phi)+\mathbb{E}_{\mu_0}[\mathbb{E}_S[s_{\mu}(\phi(\theta),x,y)]^\top\tau(\theta) ] \\
&+ H_\phi(\tau) + o(\|\tau\|_{L^2(\mu_0)}^2), 
\end{align*}
where $\mu=\phi_\sharp \mu_0$ and $H_\phi(\tau)=O(\|\tau\|_{L^2(\mu_0)}^2)$.
Thus, this formula indicates $\mathcal{L}_S(\phi)$ is Fr\'{e}chet differentiable with respect to $\phi$.
We can see its differential is represented by $\mathbb{E}_S[s_{\mu}(\phi(\theta),x,y)]$ and $s_{\mu}(\phi(\theta),x,y)$ is the stochastic gradient via $L^2(\mu_0)$-inner product.
Therefore, we can perform a stochastic variant of the functional gradient method \cite{luenberger1969optimization} to minimize $\mathcal{L}_S(\phi)$ on $L^2(\mu_0)$ and its update rule becomes as follows:
\begin{equation*}
  \phi^+ \leftarrow \phi - \eta s_{\mu}(\phi(\cdot),x,y) = (id - \eta s_{\mu}(\cdot,x,y))\circ \phi.
\end{equation*}
We immediately notice the equivalence between this update and Algorithm \ref{SPGD}, so SPGD method is nothing but the stochastic functional gradient method if the initial distribution $\mu_0$ is fixed.
However, we note that to consider the problem with respect to a probability measure $\mu$ is important because it can lead to a much better understanding of the problem as seen before.

\end{document}